\documentclass{article}

\usepackage[english]{babel}

\usepackage[letterpaper,top=2cm,bottom=2cm,left=3cm,right=3cm,marginparwidth=1.75cm]{geometry}

\usepackage{microtype} 

\usepackage[font=small,labelfont=bf]{caption}
\usepackage{enumitem}
\setlist{topsep=2pt, itemsep=1pt, parsep=0pt, partopsep=0pt, leftmargin=*}

\usepackage{amsmath,amsfonts,bm}

\def\eqref#1{equation~\ref{#1}}

\def\1{\bm{1}}

\DeclareMathAlphabet{\mathsfit}{\encodingdefault}{\sfdefault}{m}{sl}
\SetMathAlphabet{\mathsfit}{bold}{\encodingdefault}{\sfdefault}{bx}{n}

\usepackage{tabularx}
\newcolumntype{L}{>{\raggedright\arraybackslash}X}
\usepackage{adjustbox}
\newcolumntype{L}{>{\raggedright\arraybackslash}X}
\usepackage{hyperref}
\usepackage{amssymb}
\usepackage{microtype}        
\usepackage{enumitem}         
\setlist{nosep,leftmargin=*}
\usepackage{wrapfig}

\usepackage{titlesec}
\titlespacing*{\section}{0pt}{1.0ex}{0.6ex}
\titlespacing*{\subsection}{0pt}{0.8ex}{0.4ex}
\titlespacing*{\subsubsection}{0pt}{0.6ex}{0.3ex}
\usepackage{medmath}
\usepackage[noend]{algpseudocode}
\usepackage{algorithm}
\usepackage{pgfplots}
\usepackage{pgfplotstable}
\usepackage{tikz}
\pgfplotsset{compat=1.18}

\usepgfplotslibrary{groupplots}

\algrenewcommand\algorithmicindent{0.8em}
\algrenewcommand\algorithmicrequire{\textbf{Input:}}
\algrenewcommand\algorithmicensure{\textbf{Output:}}

\usepackage{caption}

\usepackage{pifont}

\usepackage{adjustbox}
\usepackage{wrapfig}    
\usepackage{algorithm}
\usepackage{algpseudocode}

\usepackage{booktabs,array}
   \usepackage{algorithm,algpseudocode} 

\usepackage{tikz}
\usetikzlibrary{calc,positioning,arrows.meta,fit,shapes.geometric}
\usetikzlibrary{arrows.meta, positioning, shadows.blur, calc}
\usepackage{helvet}  
\usepackage{courier}  
\usepackage{graphicx} 
\usepackage{natbib}  
\usepackage{caption} 
\usepackage{pgfplots}
\usepgfplotslibrary{groupplots} 
\pgfplotsset{compat=1.18}      
\usepackage{soul}
\usepackage[utf8]{inputenc}
\usepackage{caption}
\usepackage{subcaption}
\usepackage{graphicx}
\usepackage{amsmath}
\usepackage{amsthm}
\usepackage{booktabs}

\usepackage[switch]{lineno}
\usepackage{times} 
\usepackage{nccmath}
\usepackage{amsmath,amssymb,amsfonts,bm}

\newtheorem{lemma}{Lemma}

\newtheorem{proposition}{Proposition}

\newtheorem{corollary}{Corollary}
\usepackage{newfloat}
\usepackage{listings}
\DeclareCaptionStyle{ruled}{labelfont=normalfont,labelsep=colon,strut=off} 
\floatstyle{ruled}
\newfloat{listing}{tb}{lst}{}
\floatname{listing}{Listing}

\usepackage{enumitem}
\newlist{tightdesc}{description}{1}
\setlist[tightdesc]{style=nextline,font=\bfseries,leftmargin=0.9em,
  labelsep=0.35em,nosep,topsep=0pt,parsep=0pt,itemsep=0pt}

\usepackage{bibentry}

\title{Context Dependence and Reliability in Autoregressive Language Models.}

\author{
Poushali Sengupta \\
Department of Informatics, University of Oslo \\
Oslo, Norway \\
\texttt{poushals@ifi.uio.no}
\and
Shashi Raj Pandey \\
Department of Electronic Systems, Aalborg University \\
Aalborg, Denmark \\
\texttt{srp@es.aau.dk}
\and
Sabita Maharjan \\
Department of Informatics, University of Oslo \\
Oslo, Norway \\
\texttt{sabita@ifi.uio.no}
\and
Frank Eliassen \\
Department of Informatics, University of Oslo \\
Oslo, Norway \\
\texttt{frank@ifi.uio.no}
}

\begin{document}
\maketitle
\begin{abstract}
Large language models (LLMs) generate outputs by conditioning on extensive context, often containing redundant instructions, repeated facts, retrieved passages, and interaction history. In critical applications, it is essential to identify which context elements, prompts, retrievals, prior interactions, and memory, are truly used to generate an output, and how the model’s reasoning mediates their influence. This is key to building trustworthy AI systems. However, standard explanation methods, such as attention, gradients, or perturbation, struggle with redundancy, especially when memory introduces overlapping or outdated information. Even minor input changes, such as rewording a past instruction or reordering retrieved content, can cause attribution scores to shift unpredictably without changing the output, undermining interpretability in memory- and history-rich settings. Such sensitivity not only weakens explanation reliability but also raises concerns for emerging risks like prompt injection, where subtle context manipulation can hijack model behavior under the appearance of legitimate prior content. In our research, we investigate a key issue related to the trustworthiness of autoregressive LLMs: the ability to differentiate between context elements that are essential for output generation and those that are merely associated with it. We propose that dependence awareness, explicitly accounting for how each context element contributes relative to others, is critical for generating reliable explanations in LLMs. To implement this, we introduce RISE (Redundancy-Insensitive Scoring of Explanation), which quantifies the unique influence of each input element conditioned on the rest of the context. This design naturally suppresses redundant contributions and yields clearer attributions. Experiments on redundant prompting and retrieval-augmented generation (RAG) show that RISE produces significantly more stable and robust explanations than traditional perturbation-based methods. These findings highlight the role of conditional information as a principled and efficient foundation for trustworthy explanation and for monitoring LLM behavior during inference.
\end{abstract}

\section{Introduction}
LLMs are powerful tools for understanding and generating human language. They are increasingly used in important areas such as medical decision-making, legal advice, and automated systems that assist in various tasks \cite{bommasani2021opportunities, openai2023gpt4}. At their core, LLMs operate by predicting the next word (or token) based on an expanding context, including system instructions, user questions, previous conversations, and additional resources \cite{brown2020language}. As these models become more complex and are given more autonomy, it's crucial to ensure that they can be understood and trusted \cite{doshi2017accountability}. Clarifying these challenges helps the audience feel their expertise is vital in addressing trust issues. However, explaining how these models arrive at their conclusions is challenging. Current methods, such as visualizing which words the model attends to or analyzing changes in output with input adjustments, often yield inconsistent results \cite{wiegreffe2019attention, sundararajan2017axiomatic}. One reason for this is that LLMs work with long contexts that often repeat similar information. For example, instructions can be restated, documents can have similar meanings, and previous discussions can be reiterated \cite{liu2023lost}. Changes in how input is structured can lead to very different explanations, even if the overall response remains the same. This inconsistency can undermine trust and complicate the process of monitoring and evaluating the model’s performance \cite{jacovi2021formalizing}.A simple real-life analogy is getting directions from several people, where some repeat the same advice and others add a conflicting suggestion. Imagine someone asks how to reach a café:
\begin{itemize}
\item \textbf{Person A} says: `Walk straight for two blocks, then turn left.''
    \item \textbf{Person B} repeats the same advice: `Yes-straight ahead, then left.''
\item \textbf{Person C} says something different: `Go straight and take the second right.''
\end{itemize}
The person still turns left, following the information in A (and redundantly in B). Now suppose someone later asks, `Which advice caused the left turn?'' If an explanation claims that \emph{A, B, and C all influenced the decision equally}, it is clearly wrong. B adds no new information beyond A, and C is contradictory. The problem is therefore not redundancy itself, but \textbf{misassigned credit}: it can create the false impression that multiple independent sources were necessary, and it may even imply that incorrect advice contributed to a correct action. Worse, if C's instruction is written down and repeated later (e.g., as `memory''), an auditor might mistakenly conclude that `turn right'' was part of the reasoning-or that it was safe to follow-simply because the explanation system treated it as influential. In contrast, a \emph{redundancy-insensitive} explanation would behave differently. Since B does not add any new information beyond A, it should receive near-zero attribution, and the explanation should primarily credit A for the left turn. Moreover, \textbf{C should receive (approximately) zero attribution for a different reason}: it does not help explain the left turn \emph{given the rest of the context}. Once A (and B) already specify ``turn left,'' C's conflicting instruction provides no additional predictive information about the observed action; conditioning on A, the decision is already determined, so C has no unique contribution. In other words, C is not merely redundant-it is \emph{irrelevant (or negatively aligned)} with the decision, and a dependence-aware explainer should therefore suppress it rather than reward its presence.

This distinction matters in practice. First, it avoids a false ``many sources agreed'' illusion: crediting repeated statements can incorrectly suggest robustness, even though the information was merely duplicated~\cite{jacovi2020towards}. Second, it improves stability: redundancy-sensitive attributions can shift under minor paraphrases or reordering despite an unchanged output, whereas dependence-aware attributions remain consistent~\cite{kindermans2019reliability}. Third, it reduces adversarial risk: if malicious instructions are injected into retrieved context or memory and then repeated, redundancy-blind explainers may over-credit them, lending unwarranted legitimacy to manipulated behavior~\cite{wei2023jailbroken}. In short, suppressing redundancy does not remove essential information; it prevents importance from scaling with repetition, yielding explanations that reflect \emph{unique reliance} rather than frequency. More broadly, when multiple context elements convey overlapping information, attribution methods that distribute credit across inputs can \emph{systematically overvalue} redundant content. This can \textbf{mislead} users about how the model operates, creating a false impression of reliability or control~\cite{miller2019explanation}. In practice, this degrades auditability (the true decision pathway is obscured), misplaces trust (explanations appear stable while being fragile), and increases risk under prompt-injection or memory-poisoning scenarios (irrelevant content is assigned non-trivial influence)~\cite{li2023prompt}.
\begin{wraptable}{l}{0.40\linewidth}
\centering
\caption{
Comparison of explanation approaches for autoregressive LLMs.
}
\scriptsize
\begin{adjustbox}{width=\linewidth}
\begin{tabular}{
p{0.22\linewidth}
p{0.12\linewidth}
p{0.16\linewidth}
p{0.14\linewidth}
p{0.16\linewidth}
}
\toprule
\textbf{Method} 
& \textbf{Red.-aware} 
& \textbf{Stable} 
& \textbf{Context-level} 
& \textbf{Inference-time} \\
\midrule
Attention viz. 
& $\times$ & $\times$ & $\times$ & $\checkmark$ \\
Gradient saliency 
& $\times$ & $\times$ & $\times$ & $\checkmark$ \\
Perturbation 
& $\times$ & $\times$ & $\checkmark$ & $\times$ \\
Mechanistic probing 
& $\checkmark$ & $\checkmark$ & $\times$ & $\times$ \\
\midrule
\textbf{RISE (ours)} 
& $\checkmark$ & $\checkmark$ & $\checkmark$ & $\checkmark$ \\
\bottomrule
\end{tabular}
\end{adjustbox}
\label{tab:comparison}
\vspace{-1.0em}
\end{wraptable}
Large Language Models (LLMs) face significant challenges in providing clear explanations of their decision-making processes. One major issue is the difficulty of distinguishing between what is genuinely essential to a given outcome and what is merely present alongside it. In many cases, prompts, retrieved passages, and memorized information may repeat or rephrase the same facts. If an explanation method assigns equal importance to every instance of repetition, it can inaccurately imply that the model utilized a variety of independent signals or, even worse, that it based its conclusions on irrelevant or incorrect information. This can lead to three specific failure modes: 1. Over-crediting Distractors: This occurs when the model is praised for dependence on information that it did not genuinely rely on, obscuring the true sources of its output. 2. Instability: Small changes, such as paraphrasing or rearranging the order of information, can lead to significant shifts in attribution without any real change in the model's output. 3. Susceptibility to Manipulation: This is particularly relevant in scenarios such as prompt injection, where harmful instructions inserted into context or memory can unduly influence the model’s output because the explanation method used does not differentiate between actual dependencies and mere proximity of information. Current methodologies, such as attention-based scores \cite{jain2019attention,serrano2019attention}, gradient attributions \cite{sundararajan2017axiomatic,simonyan2014deep}, and simple ablation or perturbation techniques \cite{li2016understanding,feng2018pathologies}, struggle to cope with these issues. They often focus too much on the surface form of the data and do not adequately consider semantic relationships or the conditional dependencies among different pieces of contextual information. To tackle the challenge of providing reliable explanations for autoregressive language models, we propose a structured approach that prioritizes what we call``dependence awareness." This principle serves as a foundation for generating trustworthy attributions in autoregressive systems. Our contribution is a new explanation method called \textbf{RISE (Redundancy-Insensitive Scoring of Explanation)}. This method assesses the unique contribution of each context element while accounting for the influence of other elements. By diminishing the influence of overlapping or irrelevant inputs, RISE aims to deliver clearer and more accurate explanations. This, in turn, enables users and evaluators to discern the model's true dependencies, enhancing its resistance to misinterpretation and potential manipulation.

The central idea of this work is that explanation stability in autoregressive language models should be evaluated through the lens of \textit{conditional dependence}: a context element should receive attribution only to the extent that it provides information about the next-token distribution that is \emph{not} recoverable from the remaining context. Using controlled variations of prompts and retrieved context, we show that our formulation reduces redundancy-driven inflation of importance and exposes limitations of existing stability-oriented heuristics. Moreover, our approach supports structured context units, such as prompt segments, retrieved passages, or dialogue turns, which enables practical, real-time monitoring and auditing of context reliance during inference. Our main contributions in this work are:
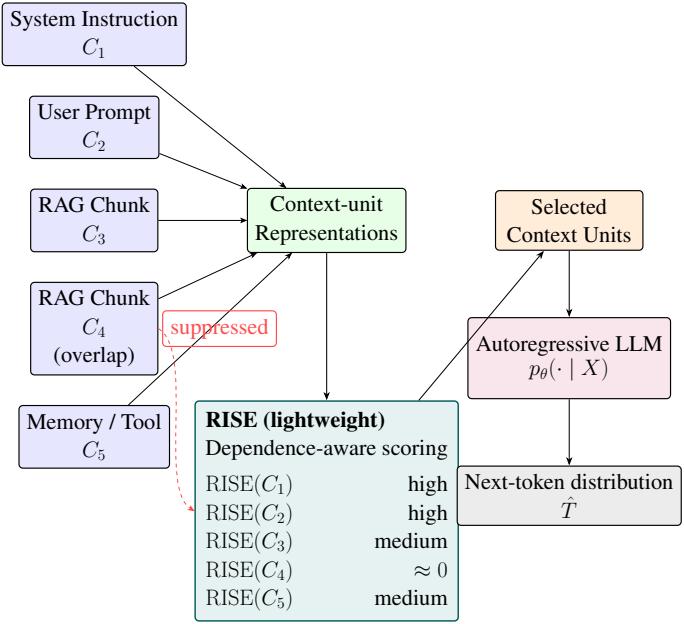
\begin{wrapfigure}{l}{0.58\linewidth}
\centering
\begin{adjustbox}{width=\linewidth}

\begin{tikzpicture}[
    font=\LARGE,
    >=Stealth,
    node distance=8mm and 14mm,
    every node/.style={
        draw,
        rounded corners=3pt,
        line width=0.9pt,
        inner sep=6pt,
        align=center
    },
    ctx/.style={fill=blue!10},
    repr/.style={fill=green!10},
    sel/.style={fill=orange!15},
    llm/.style={fill=purple!10, minimum height=22mm},
    outnode/.style={fill=gray!15},
    RISE/.style={
        draw=teal!70!black,
        fill=teal!10,
        line width=1.1pt,
        inner sep=8pt,
        minimum height=48mm,
        minimum width=48mm,
        align=left
    }
]

\node[ctx] (c1) {System Instruction\\$C_1$};
\node[ctx, below=of c1] (c2) {User Prompt\\$C_2$};
\node[ctx, below=of c2] (c3) {RAG Chunk\\$C_3$};
\node[ctx, below=of c3] (c4) {RAG Chunk\\$C_4$\\(overlap)};
\node[ctx, below=of c4] (c5) {Memory / Tool\\$C_5$};

\node[repr, right=24mm of c3, minimum width=38mm] (repr)
{Context-unit\\Representations};

\node[RISE, below=40mm of repr] (RISE)
{\textbf{RISE (lightweight)}\\
Dependence-aware scoring\\[2mm]
$\mathrm{RISE}(C_1)$ \hfill high\\
$\mathrm{RISE}(C_2)$ \hfill high\\
$\mathrm{RISE}(C_3)$ \hfill medium\\
$\mathrm{RISE}(C_4)$ \hfill $\approx 0$\\
$\mathrm{RISE}(C_5)$ \hfill medium};

\node[sel, right=24mm of repr, minimum width=40mm] (sel)
{Selected\\Context Units};

\node[llm, below=18mm of sel] (llm)
{Autoregressive LLM\\[-1mm]
$p_\theta(\cdot \mid X)$};

\node[outnode, below=18mm of llm, minimum width=36mm] (out)
{Next-token distribution\\$\hat{T}$};

\draw[->] (c1) -- (repr);
\draw[->] (c2) -- (repr);
\draw[->] (c3) -- (repr);
\draw[->] (c4) -- (repr);
\draw[->] (c5) -- (repr);

\draw[->, thick] (repr) -- (RISE);
\draw[->, thick] (RISE) -- (sel);
\draw[->, thick] (sel) -- (llm);
\draw[->, thick] (llm) -- (out);

\draw[->, dashed, red!70]
    (c4.east) .. controls +(10mm,-6mm) and +(-10mm,6mm) .. (RISE.west);

\node[font=\LARGE, red!70, right=1mm of c4.east]
{suppressed};

\end{tikzpicture}
\end{adjustbox}
\caption{
Lightweight dependence-aware context design using \textbf{RISE}.
Structured context units are mapped to compact representations, scored by their unique contribution conditioned on the rest of the context, and redundant units (e.g., overlapping RAG chunks) are suppressed before attribution and generation.}
\label{fig:overview}
\end{wrapfigure}
\begin{itemize}
    \item We identify \textit{context redundancy} as a key failure mode in explaining autoregressive LLMs: when overlapping inputs are not accounted for, attribution can appear plausible and even stable, while still assigning credit to redundant or irrelevant context, thereby obscuring the model's true dependencies.
    \item We propose \textit{conditional information} as a principled foundation for reliable attribution in autoregressive models, and introduce \textbf{RISE} as a dependence-aware scoring approach for evaluating the importance of structured context units (e.g., prompt segments, retrieved passages, and dialogue turns).
    \item Through controlled prompt and retrieval experiments, we demonstrate that dependence-aware attribution yields clearer and more meaningful explanations by reducing redundancy-driven over-crediting, and provides actionable signals for monitoring, debugging, and auditing LLM behavior, including under prompt injection-like context manipulations.
\end{itemize}
We intentionally focus the main manuscript on the core dependence-aware principle and its empirical consequences, while deferring auxiliary diagnostics, stress tests, and per-variant analyses to the appendix to preserve a clear workshop narrative.

\section{Background, Related Work, and Problem Gap}

The preceding section highlighted that explanation instability in autoregressive LLMs largely arises from unmodeled dependence and redundancy in long contexts. We now place this observation in the context of existing work on LLM explainability and formally articulate the remaining gap that motivates our approach. Autoregressive language models generate text according to,
\begin{equation}
    \medmath{p(x_1, \ldots, x_T) = \prod_{t=1}^{T} p(x_t \mid x_{<t}),}
\vspace{-2mm}
\end{equation}
where each token is conditioned on all previous tokens \cite{brown2020language}. In modern deployments, the conditioning context $x_{<t}$ is not a flat sequence but a structured composition of heterogeneous elements, including system instructions, user prompts, retrieved documents in retrieval-augmented generation (RAG), dialogue history, intermediate reasoning steps, tool outputs, and agent memory \cite{bommasani2021opportunities, lewis2020retrieval}. We denote this structured context as a set of context units, $\medmath{\mathcal{C} = \{C_1, C_2, \ldots, C_m\},}$ where each $C_i$ represents a semantically meaningful unit. This abstraction reflects how prompts are constructed and consumed in practice, and it provides a natural level at which explanations can be audited, compared, and acted upon during deployment. Research is increasingly focused on improving the interpretability of large language models (LLMs) \cite{doshi2017accountability, miller2019explanation}. Current methodologies for explaining LLMs fall into four main categories, including attention-based explanations, which visualize attention weights but can be misleading, as attention alone does not capture all influences among context elements \cite{wiegreffe2019attention}. Gradient-based methods measure output sensitivity to input changes,  yet they can be affected by model reparameterization and struggle to isolate contributions from interrelated inputs \cite{sundararajan2017axiomatic}. Perturbation-based approaches assess token importance by observing output changes after masking or removing tokens, which can lead to overestimating significance in lengthy contexts \cite{hooker2021explaining}. Mechanistic analyses explore internal model representations but are often computationally intensive and impractical for real-time use \cite{olah2018building}. To address these challenges, we propose a new method for evaluating context contributions while explicitly accounting for dependencies among context elements, enabling more accurate and efficient explanations. Our approach includes \textbf{RISE}, a lightweight, model-agnostic explanation tool for structured context units, along with an optional learned selector that incorporates dependence awareness within the model.

\section{Dependence-Aware Explanation Design via RISE}
\label{sec:rise_design}
Let $\mathcal{C}=\{C_1,\dots,C_m\}$ denote a set of \emph{structured context units} supplied to an autoregressive LLM. Each unit $C_i$ corresponds to a semantically coherent component of the prompt, such as a system instruction block, a user query, an in-context example, a retrieved document chunk in retrieval-augmented generation (RAG)~\cite{lewis2020retrieval}, a dialogue turn, a tool output, or an agent memory entry. Let $X$ denote the concatenated prompt obtained by serializing $\mathcal{C}$ according to a fixed template. This abstraction reflects how LLM contexts are constructed and manipulated in practice~\cite{brown2020language}, and it provides a natural level at which explanations can be consumed, audited, and acted upon during deployment. We consider multiple \emph{explanation targets}. The primary target is the next-token distribution $\hat{T}\triangleq p_\theta(\cdot \mid X)$. When needed, the same formulation applies to (i) a sampled next token $\hat{t}\sim p_\theta(\cdot \mid X)$, and (ii) a generated span $\hat{S}=(\hat{t}_1,\ldots,\hat{t}_H)$, which is particularly relevant for longer-horizon monitoring in agentic systems~\cite{yao2023react}.%
\footnote{If span-level generation is produced by iterative sampling, $\hat{S}$ can be treated as the joint random variable induced by the autoregressive rollout under the fixed context $X$.} Our objective is to quantify, for each context unit $C_i$, how much \emph{unique} information it contributes to the generation outcome once the remaining context $\mathcal{C}_{\setminus i}$ is taken into account. We define the unnormalized dependence-aware contribution of $C_i$ as its \textbf{Conditional Unique Dependence (CUD)} with the explanation target (formally, conditional mutual information):
\begin{equation}
\label{eq:cud}
\medmath{\Delta_i \triangleq \mathrm{CUD}(C_i;\hat{T}\mid \mathcal{C}_{\setminus i})
\;\triangleq\; I(C_i;\hat{T}\mid \mathcal{C}_{\setminus i}).}
\end{equation}
This quantity measures how much information $C_i$ provides about the generation that cannot be recovered from the remaining context~\cite{cover2006elements}. To enable comparison across context units, we normalize these contributions into our \textbf{RISE score} (Redundancy-Insensitive Scoring of Explanation):
\begin{equation}
\label{eq:rise}
\medmath{\mathrm{RISE}(C_i)
=
\frac{\mathrm{CUD}(C_i; \hat{T} \mid \mathcal{C}_{\setminus i})}
{\sum_{j=1}^m \mathrm{CUD}(C_j; \hat{T} \mid \mathcal{C}_{\setminus j}) + \varepsilon},}
\end{equation}
where $\varepsilon>0$ is a small constant used only to avoid division by zero in degenerate cases (e.g., when all CUDs are numerically estimated as zero). In practice, we set $\varepsilon$ to a negligible value (e.g., $10^{-12}$) and explicitly report when the denominator is close to zero.

For a generated span $\hat{S}$, we analogously define $\Delta_i^{(S)} \triangleq \mathrm{CUD}(C_i;\hat{S}\mid \mathcal{C}_{\setminus i})$ and normalize as in Eq.~\eqref{eq:rise}. This enables sentence-level or tool-call-level explanation and monitoring. High $\mathrm{RISE}(C_i)$ values indicate that a context unit provides information about the generation outcome that is not recoverable from the rest of the prompt, whereas low values indicate redundancy with $\mathcal{C}_{\setminus i}$. Importantly, RISE does not aim to explain internal neurons, attention heads, or mechanistic causal pathways. Instead, it captures \emph{model reliance under conditional dependence} at the level of structured context units. This perspective is well suited for trustworthy deployment: it yields explanations that are stable under prompt redundancy, directly interpretable by practitioners, and actionable for inference-time monitoring, auditing, and intervention.

\subsection{Theoretical Properties}
\label{sec:theory}

We now formalize key theoretical properties that justify \textbf{RISE} as a principled, dependence-aware explanation design for autoregressive language models operating on long and redundant contexts. Throughout, we assume that all random variables are defined on a common probability space and that conditional mutual information is well defined. Let $\mathcal{C}=\{C_1,\ldots,C_m\}$ denote the set of structured context units and let $\hat{T}$ denote the explanation target (e.g., next-token distribution or generated token). For each unit $C_i$, define its conditional contribution
$\Delta_i \triangleq I(C_i;\hat{T}\mid \mathcal{C}_{\setminus i})$, and its normalized attribution score,
\vspace{-2mm}
\begin{equation}
\medmath{\mathrm{RISE}(C_i)
=
\frac{\Delta_i}{\sum_{j=1}^m \Delta_j + \varepsilon}}.
\vspace{-2mm}
\end{equation}
If a context unit adds no new information beyond what is already present, a dependence-aware explanation should assign it zero importance.
\begin{lemma}[Redundancy suppression under conditional independence]
\label{lem:cind}
If a context unit $\medmath{C_i}$ satisfies
$\medmath{C_i \perp\!\!\!\perp \hat{T} \mid \mathcal{C}_{\setminus i},}$
then
$\medmath{\Delta_i = 0}$
and hence
$\medmath{\mathrm{RISE}(C_i) = 0}$
for $\medmath{\varepsilon = 0}$ and $\medmath{\sum_j \Delta_j > 0}$.
\end{lemma}

\begin{wraptable}{r}{0.48\linewidth}
\centering
\caption{Conceptual comparison between marginal attribution methods and \textbf{RISE} for autoregressive LLM contexts.}
\scriptsize
\begin{adjustbox}{width=\linewidth}
\begin{tabular}{lcc}
\toprule
\textbf{Aspect} & \textbf{Marginal Attribution} & \textbf{RISE (Ours)} \\
\midrule
Handles redundancy & No & Yes (by conditioning) \\
Stability under prompt variation & Low & High \\
Context-unit explanations & Limited & Native \\
Inference-time monitoring & Difficult & Direct \\
\bottomrule
\end{tabular}
\end{adjustbox}
\label{tab:rise_vs_marginal}
\end{wraptable}

A context element that does not provide unique information about the model output, once all other context is known, should receive zero attribution. Duplicating the same instruction or context should not increase its importance, as no additional information is introduced.
\begin{corollary}[Duplicate-context invariance]
\label{cor:duplicate}
Suppose two context units $C_i$ and $C_i'$ are exact duplicates, i.e.,
$\medmath{C_i = C_i'}$ almost surely.
Then, in the ideal population setting,
$\medmath{I(C_i;\hat{T}\mid \mathcal{C}_{\setminus i}) = 0}$
and
$\medmath{I(C_i';\hat{T}\mid \mathcal{C}_{\setminus i'}) = 0,}$
and both units receive zero \textbf{RISE} attribution.
\end{corollary}
Repeating the same prompt or instruction does not make it more important; \textbf{RISE} correctly ignores duplicated context. Adding redundant context should not change the importance assigned to existing, informative context units.
\begin{proposition}[Stability under redundancy injection]
\label{prop:stability}
Let $\medmath{\mathcal{C}' = \mathcal{C} \cup \{C'\}}$, where $C'$ is redundant given $\medmath{\mathcal{C}}$, i.e.,
$\medmath{C' \perp\!\!\!\perp \hat{T} \mid \mathcal{C}.}$
Then, for all original units $C_i \in \mathcal{C}$,
$\medmath{\mathrm{RISE}_{\mathcal{C}'}(C_i) = \mathrm{RISE}_{\mathcal{C}}(C_i),}$
in the ideal population setting.
\end{proposition}
\begin{wraptable}{r}{0.48\linewidth}
\centering
\scriptsize
\caption{Deployment-oriented benefits of lightweight dependence-aware context selection (mandatory recent window + unique anchors).}
\begin{adjustbox}{width=\linewidth}
\begin{tabular}{p{0.28\linewidth}p{0.66\linewidth}}
\toprule
\textbf{Design outcome} & \textbf{What it enables in deployment} \\
\midrule
Reduced context size ($L+K \ll m$) 
& Lower latency and reduced variance by eliminating redundant context units while preserving essential information. \\
Dependence-aware anchors 
& Retention of long-range or external information without importance splitting caused by redundancy. \\
Stable explanations 
& Robust attributions under redundancy injection, paraphrasing, and minor context shifts. \\
Monitoring signal 
& Drift in selected anchor sets can indicate retrieval changes, prompt path deviations, or distribution shift at inference time. \\
\bottomrule
\end{tabular}
\end{adjustbox}
\label{tab:light_benefits}
\end{wraptable}
Injecting redundant context does not distort explanations for the original inputs; \textbf{RISE} preserves attribution stability by design. Additional properties, extensions, and stress-test results are deferred to Appendix.
\textbf{RISE} provides dependence-aware \emph{attribution}, but the same conditional-information principle can also be used to design a \emph{lightweight context selector} that reduces the effective input dimension of an autoregressive system. The key observation is that autoregressive contexts are often highly redundant: in time series, many historical lags encode overlapping seasonal/inertial information; in LLMs, prompt segments and retrieved chunks frequently paraphrase or overlap semantically. Feeding all context to the model can dilute fast-changing signals, increase variance, and destabilize explanations. Instead of truncating context or selecting units by marginal importance, we select a small subset that maximizes \emph{unique} information after conditioning on a mandatory ``recent'' context. This yields a compact context that preserves responsiveness to recent changes while retaining a few long-range anchors (e.g., weekly periodicity lags or key retrieved evidence).

Let $\medmath{\mathcal{C}=\{C_1,\ldots,C_m\}}$ be ordered context units, where the ordering corresponds to recency (time-series lags) or prompt position (LLM serialization). We define: $\medmath{\mathcal{R} \triangleq \{C_1,\ldots,C_L\}}$ \text{(mandatory recent window)},
$\medmath{\mathcal{H} \triangleq \{C_{L+1},\ldots,C_m\}}$ (candidates). The lightweight selector outputs a reduced context,
$\medmath{\tilde{\mathcal{C}} \;=\; \mathcal{R} \cup \mathcal{S}, \ \mathcal{S}\subset \mathcal{H},\ |\mathcal{S}|=K,}$
where $\medmath{\mathcal{S}}$ contains the top-$K$ context units that add the most \emph{unique} information beyond $\mathcal{R}$. For each candidate $\medmath{C_i\in \mathcal{H}}$, we define a dependence-aware score conditioned on the mandatory recent window:
\begin{equation}
\label{eq:light_score}
\medmath{\Delta_i^{(\mathcal{R})} \triangleq I(C_i;\hat{T}\mid \mathcal{R}).}
\end{equation}
Equivalently, we can use the corresponding normalized \textbf{RISE-restricted} score:
\begin{equation}
\label{eq:light_rise}
\medmath{\mathrm{RISE}_{\mathcal{H}\mid \mathcal{R}}(C_i)
=
\frac{I(C_i;\hat{T}\mid \mathcal{R})}{\sum_{C_j\in\mathcal{H}} I(C_j;\hat{T}\mid \mathcal{R})+\varepsilon}.}
\end{equation}
We then select $\medmath{\mathcal{S}}$ as the top-$K$ units in $\medmath{\mathcal{H}}$ by $\medmath{\Delta_i^{(\mathcal{R})}}$ (or $\mathrm{RISE}_{\mathcal{H}\mid \mathcal{R}}$).
If $\medmath{C_i = y_{t-i}}$ denotes lag-$i$, then $\mathcal{R}$ corresponds to the last $L$ lags and $\mathcal{S}$ selects a few long-range anchors (e.g., $t-24$, $t-168$) that remain informative \emph{after conditioning on the recent window}. This yields a reduced feature vector of size $L+K$ instead of $m$.
If $C_i$ is a prompt segment or retrieved chunk, then $\mathcal{R}$ corresponds to the most recent dialogue turns/instructions (or the highest-priority control blocks), while $\mathcal{S}$ retains a small number of uniquely informative retrieval/memory units whose information is not recoverable from $\mathcal{R}$.
\begin{lemma}[Conditional redundancy implies non-selection (ideal case)]
\label{lem:nonselect}
If a candidate unit $C_i\in\mathcal{H}$ is conditionally independent of the target given the recent window, i.e., $\medmath{C_i \perp\!\!\!\perp \hat{T}\mid \mathcal{R},}$ then $I(C_i;\hat{T}\mid \mathcal{R})=0$. Hence $C_i$ cannot appear in the top-$K$ set $\mathcal{S}$ unless fewer than $K$ candidates have positive scores.
\end{lemma}
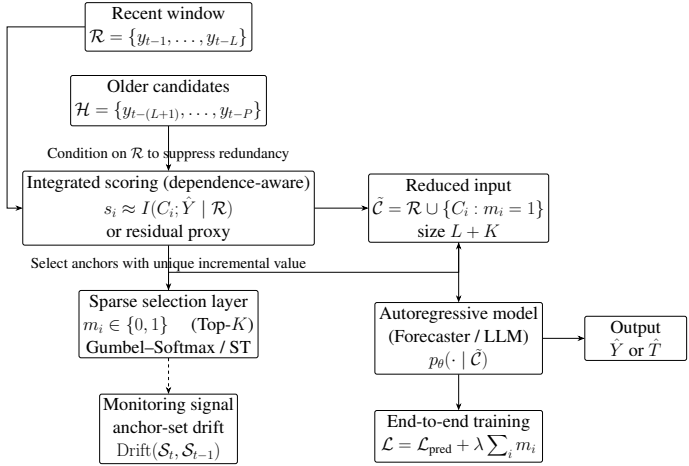
\begin{wrapfigure}{l}{0.58\linewidth}
\centering
\begin{adjustbox}{width=\linewidth}
\begin{tikzpicture}[
    font=\LARGE,
    >=Stealth,
    node distance=9mm and 16mm,
    box/.style={draw, rounded corners=2pt, inner sep=4pt, align=center},
    smallbox/.style={draw, rounded corners=2pt, inner sep=3pt, align=center},
    arrow/.style={->, line width=0.8pt},
    dashedarrow/.style={->, dashed, line width=0.8pt}
]

\node[box, minimum width=36mm] (recent)
{Recent window\\
$\mathcal{R}=\{y_{t-1},\ldots,y_{t-L}\}$};

\node[box, below=8mm of recent, minimum width=36mm] (cand)
{Older candidates\\
$\mathcal{H}=\{y_{t-(L+1)},\ldots,y_{t-P}\}$};

\node[box, below=15mm of cand, minimum width=50mm, minimum height=16mm] (score)
{Integrated scoring (dependence-aware)\\[1mm]
$s_i \approx I(C_i;\hat{Y}\mid\mathcal{R})$\\
or residual proxy};

\node[box, below=15mm of score, minimum width=50mm, minimum height=16mm] (mask)
{Sparse selection layer\\
$m_i\in\{0,1\}$ \quad (Top-$K$)\\
Gumbel--Softmax / ST};

\node[box, above=8mm of mask, right=18mm of score, minimum width=46mm] (reduced)
{Reduced input\\
$\tilde{\mathcal{C}}=\mathcal{R}\cup\{C_i:m_i=1\}$\\
size $L+K$};

\node[box, below=20mm of reduced, minimum width=44mm, minimum height=16mm] (model)
{Autoregressive model\\
(Forecaster / LLM)\\
$p_\theta(\cdot\mid\tilde{\mathcal{C}})$};

\node[box, right=14mm of model, minimum width=34mm] (out)
{Output\\
$\hat{Y}$ or $\hat{T}$};

\node[box, below=12mm of model, minimum width=56mm] (loss)
{End-to-end training\\
$\mathcal{L}=\mathcal{L}_{\text{pred}}+\lambda\sum_i m_i$};

\node[smallbox, below=12mm of mask, minimum width=46mm] (monitor)
{Monitoring signal\\
anchor-set drift\\
$\mathrm{Drift}(\mathcal{S}_t,\mathcal{S}_{t-1})$};

\draw[arrow] (recent.west) -- ++(-26mm,0mm) |- (score.west);
\draw[arrow] (cand.south) -- (score.north);

\draw[arrow] (score.east) -- (reduced.west);
\draw[arrow] (score.south) -- (mask.north);

\draw[arrow] (mask.north) -- ++(0,6mm) -| (reduced.south);

\draw[arrow] (reduced.south) -- (model.north);
\draw[arrow] (model.east) -- (out.west);

\draw[arrow] (model.south) -- (loss.north);

\draw[dashedarrow] (mask.south) -- (monitor.north);

\node[font=\Large, above=2mm of score]
{Condition on $\mathcal{R}$ to suppress redundancy};

\node[font=\Large, above=5mm of mask]
{Select anchors with unique incremental value};

\end{tikzpicture}
\end{adjustbox}
\caption{Lightweight dependence-aware selection: retain recent context for fast dynamics and sparsely select uniquely informative long-range anchors, enabling inference-time monitoring via anchor drift.}
\label{fig:integrated_lightweight}
\end{wrapfigure}
In simple terms, if a candidate context unit does not contribute any information beyond what is already present in the recent context, \textbf{RISE} will never select it unless there are no better alternatives. Duplicated or highly overlapping context should not receive additional importance simply because it appears multiple times.
\begin{corollary}[Duplicate/overlap suppression under conditioning]
\label{cor:overlap}
If two candidates $C_i$ and $C_j$ are near-duplicates in the sense that $C_i$ is (approximately) a deterministic function of $(\mathcal{R},C_j)$, then $I(C_i;\hat{T}\mid \mathcal{R})$ is (approximately) zero whenever $C_j$ is available. Thus overlap in retrieval chunks or long-range seasonal redundancy among lags is naturally suppressed.
\end{corollary}
This means that when two retrieved chunks or context units convey essentially the same information, \textbf{RISE} attributes importance to at most one of them, preventing redundancy-induced importance inflation. A good selector should remain sensitive to rapidly changing information while still retaining a small amount of relevant long-range context.
\subsection{Computational Complexity and Estimation}
\label{sec:complexity}
Computing \textbf{RISE} requires estimating $m$ conditional mutual information terms of the form
$I(C_i;\hat{T}\mid\mathcal{C}_{\setminus i})$.
Let $n$ denote the number of evaluated prompts or generation steps, and let $d$ denote the dimensionality of the chosen context-unit representation (e.g., pooled hidden states, or logit vectors). The total computational cost can be expressed as $\mathcal{O}\!\left(m \cdot \mathrm{CostCMI}(n,d)\right),$
where $\mathrm{CostCMI}(n,d)$ depends on the chosen conditional mutual information estimator (e.g., kNN-based, variational, or kernel-based). In practice, \textbf{RISE} remains lightweight for three key reasons. First, attribution operates over \emph{structured context units} rather than raw tokens, keeping $m$ small even for long prompts.
 \begin{wraptable}{l}{0.48\linewidth}
\centering
\tiny
\caption{Failure modes under redundancy and expected diagnostic signatures (output preserved).}
\label{tab:failure_modes}
\begin{adjustbox}{width=\linewidth}
\begin{tabular}{p{0.25\linewidth}p{0.38\linewidth}p{0.30\linewidth}}
\toprule
\textbf{Failure mode} & \textbf{What the baseline tends to do} & \textbf{Expected signature (diagnostic)} \\
\midrule
Redundancy blindness 
& Splits or duplicates credit across repeated / overlapping units (marginal scoring). 
& Attribution mass on duplicates \emph{grows or fragments} while the model output is unchanged. \\
\addlinespace
Surface-form sensitivity 
& Changes scores under paraphrase/reorder despite semantic equivalence. 
& Rank/overlap changes \emph{without any output change}. \\
\addlinespace
Structural dominance 
& Collapses onto structurally necessary units (e.g., question/template) under perturbation. 
& Removing the question causes a large log-prob drop regardless of evidence $\Rightarrow$ attribution collapses to Q. \\
\bottomrule
\end{tabular}
\end{adjustbox}
\end{wraptable}
Second, units can be grouped (e.g., retrieved chunks or dialogue windows) to reduce dimensionality. Third, with cached context representations and next-token distributions, \textbf{RISE} avoids re-running the LLM per unit, enabling lightweight pre-model or post-hoc inference-time auditing.

\subsection{Integrated Learned Selector}
\label{sec:learned_selector}
We implement the dependence-aware principle behind \textbf{RISE} in two complementary ways: (i) a lightweight, model-agnostic \textbf{RISE}-Lite scoring module for structured context units, and (ii) an \emph{integrated learned selector} that enforces dependence-aware context retention \emph{within} the generation pipeline. In the learned-selector variant, a small gating network operates on context-unit representations $\mathbf{z}_i$ and produces selection scores that are converted into binary retention decisions via standard differentiable relaxations. The generator then conditions only on the retained subset of context units. The selector and generator are trained end-to-end with a sparsity regularizer to encourage reliance on a minimal context set. Compared to \textbf{RISE}-Lite, this integrated selector introduces additional training cost and hyperparameters, but provides tighter pipeline integration and enables amortized dependence-aware selection at inference time. We use \textbf{RISE}-Lite for transparent auditing and monitoring, and the integrated selector when end-to-end context reduction is desired.

\subsection{Design Principle: Dependence-Aware Explanations for Trustworthy AI}
\label{sec:design_principle}
We present a key design principle for trustworthy explanations in AI systems, particularly LLMs: \emph{explanations must be dependence-aware}, reflecting each component's \emph{unique} contribution to model behavior rather than correlation or frequency. This shifts attribution from surface associations to true reliance. Traditional methods misestimate importance under redundancy or overlap; dependence-aware explanations value only information not recoverable from other components, preventing credit for repetition. A central requirement is \emph{redundancy invariance}: repeating the same content must not increase its attributed importance. Since real deployments often filter or truncate context, trustworthy explanations should remain valid under context selection; dependence-aware attribution preserves redundancy suppression even when computed on a subset. Moreover, dependence-aware explanations rely only on observable input-output behavior, requiring no access to internal mechanisms, which makes them suitable for monitoring and auditing. Dependence awareness also improves inference-time safety: failures often arise from interactions among prompts, retrieval, and memory, not isolated inputs. \textbf{RISE} can surface issues such as retrieval collapse (redundant sources), prompt drift (inconsistent instructions), and memory redundancy (repeated states) using only input-output traces, making it practical for safety-critical settings. Importantly, redundancy is often \emph{structural} in safe system design: retrieval may return overlapping evidence and agent memories may duplicate state for reliability. Explanations that ignore this can misrepresent the system, motivating methods that respect valuable context while suppressing repetition. For example, if a safety instruction is repeated five times verbatim, the model's behavior typically does not change after the first occurrence, yet marginal attribution may assign roughly five times more importance due to repetition alone. Such explanations measure frequency, not dependence; a trustworthy explanation should assign the same importance regardless of how many times the same information is restated.
 \begin{table}[htbp]
\centering
\tiny
\caption{\textbf{Baseline failure $\rightarrow$ RISE fix.}
Top: SQuAD redundancy/paraphrasing (20). Bottom: GPT-2 output-preserving tests ($n=30$, $k=2$, overlap $=0.7$).}
\label{tab:rise_merged}
\vspace{-1mm}
\begin{adjustbox}{width=\linewidth}
\begin{minipage}{\linewidth}
\centering

\begin{tabular}{
p{0.20\linewidth}
p{0.22\linewidth}
p{0.18\linewidth}
p{0.20\linewidth}
}
\toprule
\textbf{Method} 
& \textbf{Rank Stability (Spearman $\rho$)} $\uparrow$ 
& \textbf{Top-$5$ Overlap} $\uparrow$ 
& \textbf{Dup-Split} $\downarrow$ \\
\midrule
Attention      
& $0.68 \pm 0.03$ 
& $0.88 \pm 0.02$ 
& $0.29 \pm 0.01$ \\
RISE (ours)    
& $0.31 \pm 0.04$ 
& $0.75 \pm 0.03$ 
& $\mathbf{0.50 \pm 0.00}$ \\
Perturbation   
& $\mathbf{0.81 \pm 0.05}$ 
& $0.83 \pm 0.04$ 
& $0.73 \pm 0.06$ \\
\bottomrule
\end{tabular}

\vspace{2mm}

\begin{tabular}{lcc}
\toprule
\textbf{Method} & \textbf{Faithfulness Gap} $\uparrow$ & \textbf{Overlap SplitIndex} $\downarrow$ \\
\midrule
Attention & $-3.15 \pm 6.77$ & $0.8445 \pm 0.0914$ \\
Perturbation (LOUO) & $3.79 \pm 5.48$ & $0.0139 \pm 0.0749$ \\
RISE (ours) & $3.79 \pm 5.48$ & $0.0139 \pm 0.0749$ \\
\bottomrule
\end{tabular}

\end{minipage}
\end{adjustbox}
\vspace{-2mm}
\end{table}
\section{Experimental Setup}
\label{sec:experiments}
We conduct experiments to investigate whether explanation methods avoid inflated credit to redundant or dominant context when model output remains unchanged. Our aim is not to improve accuracy but to stress-test explanation behavior under long-context variations~\cite{bommasani2021opportunities}. Using GPT-2 on the SQuAD v1.1 validation set~\cite{radford2019language,rajpurkar2016squad} and an open-weight autoregressive model as a proxy for API-based systems (e.g., GPT-3.5)~\cite{brown2020language}, we focus on \emph{input--output} explanation behavior under a strict \textbf{output-preservation contract}, scoring explanations only when the model output is fixed. Under this contract, any attribution change is \emph{explanation instability}, not model-behavior change, making redundancy-invariance measurable.
Starting from a base prompt $X$, we introduce redundant context while keeping the answer unchanged via: (i) \textbf{duplication}, (ii) 
\begin{wrapfigure}{l}{0.75\linewidth}
    \centering
    \begin{tikzpicture}
\begin{axis}[
    width=\linewidth,
    height=4cm,
    ybar,
    bar width=14pt,
    ymin=0, ymax=1.05,
    ylabel={Mean Score},
    symbolic x coords={
        Stability (Spearman $\rho$) $\uparrow$,
        Top-5 Overlap $\uparrow$,
        Redundancy Sensitivity $\downarrow$
    },
    xtick=data,
    xticklabel style={align=center},
    enlarge x limits=0.25,
    legend style={
        at={(0.5,-0.25)},
        anchor=north,
        legend columns=3,
        draw=none
    },
    title={Average Performance Across Methods},
    ymajorgrids=true,
    grid style={gray!30}
]

\addplot coordinates {
    (Stability (Spearman $\rho$) $\uparrow$, 0.67)
    (Top-5 Overlap $\uparrow$, 0.88)
    (Redundancy Sensitivity $\downarrow$, 0.29)
};

\addplot coordinates {
    (Stability (Spearman $\rho$) $\uparrow$, 0.31)
    (Top-5 Overlap $\uparrow$, 0.75)
    (Redundancy Sensitivity $\downarrow$, 0.50)
};

\addplot coordinates {
    (Stability (Spearman $\rho$) $\uparrow$, 0.81)
    (Top-5 Overlap $\uparrow$, 0.83)
    (Redundancy Sensitivity $\downarrow$, 0.72)
};

\legend{Attention, RISE, Perturbation}

\end{axis}
\end{tikzpicture}
    \caption{Aggregate attribution behavior across methods. \textbf{RISE} exhibits stronger redundancy suppression (lower redundancy sensitivity) despite lower raw rank stability, reflecting dependence-aware scoring under overlap.}
    \label{fig:aggregate_bars}
\end{wrapfigure}
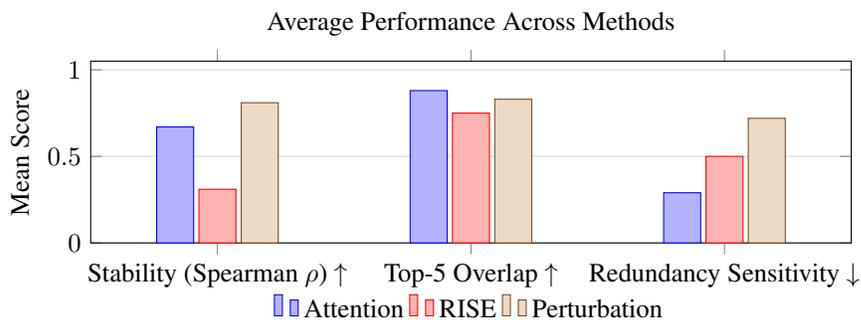\textbf{paraphrasing}, and (iii) \textbf{reordering/overlap}. These methods reflect realistic prompt engineering and memory effects. We compare \textbf{RISE} with two baseline families: \textbf{attention-based heuristics} and \textbf{perturbation-based attribution}. Under redundancy, baselines may show: (1) \emph{Redundancy blindness}, where scores inflate across duplicates; (2) \emph{Surface-form sensitivity}, where reordering affects attributions without changing outputs; and (3) \emph{Structural dominance}, where essential units disrupt generation. We evaluate using \textbf{Rank stability} (Spearman $\rho$), \textbf{Top-$K$ overlap}, and \textbf{Redundancy Sensitivity (Dup-Split)}, which measures attribution fragments under duplication. A good explainer maintains low Dup-Split with unchanged outputs, indicating \emph{redundancy invariance}. Dup-Split serves as a reliable diagnostic tool in redundancy scenarios.
\vspace{-2mm}
\section{Results}
\label{sec:results}
Our primary goal is to test whether explanation methods remain \emph{faithful under redundancy}: when the model output does not change, do attributions avoid rewarding duplicated or overlapping context, and do they avoid collapsing onto structurally dominant units?
Figures such as \autoref{fig:aggregate_bars} and \autoref{tab:rise_merged} report aggregate behavior across examples.
We observe two distinct baseline failure patterns that become prominent once prompts contain duplicates, paraphrases, or overlapping retrieval.
First, \textbf{attention-based heuristics} can appear stable (e.g., high Top-$K$ overlap), but this stability is often explained by \emph{diffusion}: importance is spread across repeated or correlated units, producing consistent-looking supports without isolating uniquely required context. In redundancy-rich prompts, attention can therefore remain ``consistent'' while still over-crediting duplicates. Second, \textbf{perturbation-based attribution} often achieves higher Spearman rank stability, but this can be misleading because perturbation 
is vulnerable to \emph{structural dominance}.
Removing the question (or similarly structural units) disrupts generation irrespective of semantic evidence, so perturbation may assign overwhelming credit to the question even when informative evidence lies in retrieved or supporting context.
Under duplication, perturbation can also inflate credit because repeated units create multiple removal points, increasing marginal effects without adding new information. RISE targets a different criterion: \emph{redundancy invariance under conditional dependence}.
By attributing only what is not recoverable from the remaining context, RISE suppresses redundancy-induced importance inflation and prevents credit splitting across near-duplicates.
These trends reflect \textbf{redundancy-aware behavior}: RISE maintains stronger redundancy suppression (lower Dup-Split) under duplication and overlap, while its lower raw rank stability is expected because it \emph{intentionally} reallocates credit away from redundant copies toward uniquely informative units, a feature under redundancy stress tests rather than a bug. In contrast, baselines can appear stable while failing to isolate unique dependence. We also observe a structural failure mode of perturbation under redundancy (Table~\ref{tab:structural_collapse}): perturbation collapses onto the question due to its structural necessity, conflating \emph{removability} with \emph{informational contribution}. RISE(-Lite), by conditioning on the remaining context, reallocates mass to context units that remain informative, yielding a compact support aligned with uniquely informative evidence. We further evaluate explanations only under \emph{output-preserving} interventions (same decoded answer/span or next-token argmax); under this contract, attribution changes indicate \emph{explanation instability}, not model-behavior change. Table~\ref{tab:rise_merged} makes the baseline failure $\rightarrow$ RISE fix explicit: \textbf{(i)} under RAG overlap, attention splits credit across overlapped chunks (high SplitIndex), whereas RISE suppresses overlap (low SplitIndex); \textbf{(ii)} in the faithfulness removal test, attention yields a negative top--bottom gap, while RISE produces a positive gap, indicating that removing top-ranked units reduces answer likelihood substantially more than removing bottom-ranked units. Per-example analyses and additional stress tests (Appendix) confirm that these effects are not driven by a small subset of instances and persist across models.
\vspace{-2mm}
\subsection{Discussion and Limitations}
\label{sec:discussion}
\textbf{RISE} provides a dependence-aware account of how autoregressive language models rely on context. Our results show that standard notions of ``explanation stability'' can be misleading under redundancy: high agreement may simply reflect \emph{importance splitting} across correlated or repeated units (diffusion), rather than reliance on uniquely informative context. By conditioning on the remaining context, \textbf{RISE} distinguishes what is essential from what is merely repeated, yielding clearer and more actionable explanations in prompt, retrieval, and memory pipelines. Table~\ref{tab:structural_collapse} highlights a complementary failure mode: perturbation-based attribution can collapse onto the question because removing it disrupts generation regardless of semantics, conflating \emph{structural necessity} with \emph{informational contribution}. 
\begin{wraptable}{l}{0.48\linewidth}
\centering
\tiny
\begin{tabular}{lcc}
\toprule
\textbf{Unit} & \textbf{RISE-Lite} & \textbf{Pert.} \\
\midrule
System (Sys)   & 0.27 & 0.00 \\
Context 1 (C1) & 0.21 & 0.00 \\
Context 2 (C2) & \textbf{0.39} & 0.08 \\
Context 3 (C3) & 0.12 & 0.00 \\
Question (Q)   & $\approx$0.00 & \textbf{0.92} \\
\bottomrule
\end{tabular}
\caption{
Representative attribution under prompt redundancy.
Perturbation collapses onto the question due to structural dominance, while RISE-Lite concentrates attribution on uniquely informative context.
}
\label{tab:structural_collapse}
\end{wraptable}
In contrast, \textbf{RISE-Lite} suppresses structurally dominant yet recoverable units and reallocates mass to context components that remain informative after conditioning. A limitation of \textbf{RISE} is that it operates over \emph{structured context units} (prompt blocks, retrieved chunks, dialogue turns) rather than individual tokens, improving interpretability and deployability at the cost of token-level granularity. Moreover, our focus is on redundancy and semantic overlap; for these goals, preserving the main conditional-dependence structure is more critical than perfectly calibrated dependence estimates. Future work will develop improved estimators, quantify approximation error, and integrate \textbf{RISE} into real-time safety and reliability pipelines for agentic and retrieval-augmented systems. 

\textbf{Ethics and Reproducibility: }
This work improves LLM transparency by mitigating misleading explanations under redundant or manipulated context, without introducing new model capabilities or ethical risks. RISE is model-agnostic, requires no retraining, and is evaluated on public models and datasets; data, models, and code are available at \url{https://anonymous.4open.science/r/LLM-Trustworthy-3448/README.md}.

\vspace{-2mm}
\section{Conclusion}
We addressed a core trustworthiness challenge in autoregressive LLMs: identifying which context elements are \emph{uniquely required} for generation under redundancy. Using conditional information as a design principle, \textbf{RISE} suppresses redundancy-induced importance inflation, remains interpretable, and supports inference-time monitoring. Across aggregate evaluation, stress tests, and cross-model diagnostics, we show that apparent stability can mask structural failures in attention- and perturbation-based heuristics. Beyond interpretability, \textbf{RISE} enables practical safeguards such as detecting retrieval collapse, prompt drift, and agent-memory redundancy at inference time, providing a foundation for trustworthy prompt auditing, retrieval validation, and monitoring in redundancy-prone LLM deployments.

\bibliography{iclr2026/iclr2026_conference}
\bibliographystyle{iclr2026_conference}

\appendix

\section{Additional Theoretical Results and Proofs}
\label{app:theory}

In the main paper, I present the core redundancy-aware guarantees of \textbf{RISE} (Lemma~3, Corollary~1, Proposition~1), as these directly motivate the empirical design and explain the observed behavior under prompt duplication and paraphrasing. In this appendix, I provide auxiliary properties (non-negativity, normalization), group-level extensions, and (optionally) the learned-selector variant. All proofs are written step-by-step to make the logical flow explicit and reproducible.

\subsection{Notation and Preliminaries}
\label{app:notation}

Let $C=\{c_1,\dots,c_m\}$ denote the set of context units (instruction segments, retrieved chunks, question, etc.). Let $Y$ denote the model output variable of interest (e.g., next-token distribution, log-probability of an answer string, or a scalar proxy derived from generation). For $i \in \{1,\dots,m\}$, let $C_{-i} := C \setminus \{c_i\}$. I write \emph{conditional unique dependence} (CUD) as
\begin{equation}
I(Y;c_i \mid C_{-i})
\;=\;
\mathbb{E}\left[\log\frac{p(Y \mid c_i, C_{-i})}{p(Y \mid C_{-i})}\right].
\end{equation}
CUD is non-negative and equals zero if and only if $Y \perp c_i \mid C_{-i}$ (conditional independence), under standard regularity assumptions.
\textbf{RISE} assigns each context unit a conditional unique-dependence contribution
\begin{equation}
\phi_i \;:=\; I(Y;c_i \mid C_{-i}),
\qquad i=1,\dots,m,
\label{eq:rise_phi_def}
\end{equation}
and uses a normalized attribution
\begin{equation}
a_i \;:=\; \frac{\phi_i}{\sum_{j=1}^m \phi_j + \varepsilon},
\label{eq:rise_norm_def}
\end{equation}
for a small $\varepsilon>0$ used only for numerical stability. (In the main text, I set $\varepsilon$ to a tiny constant and omit it for clarity.)
Lemma~3 in the main paper uses the equivalence $I(Y;c_i \mid C_{-i})=0 \Leftrightarrow (Y \perp c_i \mid C_{-i})$ to formalize redundancy suppression. The following auxiliary lemmas support the remaining properties and implementation details.

\begin{lemma}[Non-negativity]
\label{lem:nonneg}
For any context unit $c_i$, the \textbf{RISE} contribution satisfies $\phi_i = I(Y;c_i \mid C_{-i}) \ge 0$.
\end{lemma}

\begin{proof}
 Let $C=(c_i,C_{-i})$ denote the full context, where $C_{-i}$ collects all context units except $c_i$. The random variable $Y$ denotes the generation target (e.g., next-token outcome or next-token distribution), and all probabilities below are taken with respect to the joint distribution induced by the data-generating process together with the autoregressive model under evaluation. By definition, conditional unique dependence is
\begin{equation}
I(Y;c_i\mid C_{-i})
=
\mathbb{E}_{(Y,c_i,C_{-i})}\!\left[
\log
\frac{p(Y,c_i\mid C_{-i})}{p(Y\mid C_{-i})\,p(c_i\mid C_{-i})}
\right].
\label{eq:cud_def_log_ratio}
\end{equation}
Using $p(Y,c_i\mid C_{-i})=p(Y\mid c_i,C_{-i})\,p(c_i\mid C_{-i})$, the ratio in~\eqref{eq:cud_def_log_ratio} simplifies to
\begin{equation}
\frac{p(Y,c_i\mid C_{-i})}{p(Y\mid C_{-i})\,p(c_i\mid C_{-i})}
=
\frac{p(Y\mid c_i,C_{-i})}{p(Y\mid C_{-i})}.
\end{equation}
Substituting back, I obtain the standard conditional form
\begin{equation}
I(Y;c_i\mid C_{-i})
=
\mathbb{E}_{(Y,c_i,C_{-i})}\!\left[
\log
\frac{p(Y\mid c_i, C_{-i})}{p(Y\mid C_{-i})}
\right].
\label{eq:cud_def_conditional}
\end{equation}

Equation~\eqref{eq:cud_def_conditional} is an average of pointwise log-likelihood ratios. Grouping terms by conditioning variables $(c_i,C_{-i})$ turns this average into an expected KL divergence between the conditional distributions $p(Y\mid c_i,C_{-i})$ and $p(Y\mid C_{-i})$.

Condition first on $(c_i,C_{-i})$ and take expectation over $Y$:
\begin{align}
I(Y;c_i\mid C_{-i})
&=
\mathbb{E}_{(c_i,C_{-i})}\!
\left[
\mathbb{E}_{Y\mid c_i,C_{-i}}
\left[
\log \frac{p(Y\mid c_i,C_{-i})}{p(Y\mid C_{-i})}
\right]
\right] \nonumber\\
&=
\mathbb{E}_{(c_i,C_{-i})}\!
\left[
D_{\mathrm{KL}}
\!\left(
p(Y\mid c_i,C_{-i})
\,\|\, 
p(Y\mid C_{-i})
\right)
\right].
\label{eq:cud_as_expected_kl}
\end{align}

For any distributions $p$ and $q$ defined on the same support, the Kullback--Leibler divergence satisfies
\begin{equation}
D_{\mathrm{KL}}(p\|q)\ge 0,
\end{equation}
with equality if and only if $p=q$ almost surely. Applying this to each realized pair $(c_i,C_{-i})$ yields
\begin{equation}
D_{\mathrm{KL}}
\!\left(
p(Y\mid c_i,C_{-i})
\,\|\, 
p(Y\mid C_{-i})
\right)\ge 0
\qquad \text{for all }(c_i,C_{-i}).
\end{equation}

Since the integrand in~\eqref{eq:cud_as_expected_kl} is non-negative everywhere, its expectation over $(c_i,C_{-i})$ is also non-negative. Therefore,
\begin{equation}
I(Y;c_i\mid C_{-i}) \ge 0,
\end{equation}
which implies $\phi_i \ge 0$ as claimed.
\end{proof}

Lemma~\ref{lem:nonneg} guarantees that each unnormalized \textbf{RISE} component $\phi_i$ is a valid non-negative ``attribution mass'' prior to the normalization step in~\eqref{eq:rise_norm_def}; thus, normalization only rescales these masses to enable comparability across context units without changing their sign.
\begin{lemma}[Normalization]
\label{lem:norm}
Assume $\sum_{j=1}^m \phi_j > 0$. Then the normalized \textbf{RISE} attributions $a_i$ in~\eqref{eq:rise_norm_def} satisfy $a_i \ge 0$ and $\sum_{i=1}^m a_i = 1$.
\end{lemma}

\begin{proof}
By~\eqref{eq:rise_norm_def}, the normalized attribution for unit $i$ is
\begin{equation}
a_i \;=\; \frac{\phi_i}{\sum_{j=1}^m \phi_j}.
\label{eq:ai_def_repeat}
\end{equation}
The assumption $\sum_{j=1}^m \phi_j>0$ ensures that the denominator in~\eqref{eq:ai_def_repeat} is strictly positive, so the normalization is well-defined.
Since $\{\phi_i\}_{i=1}^m$ are non-negative ``masses'' (Lemma~\ref{lem:nonneg}), dividing by their positive total simply rescales them into weights that can be interpreted on a simplex. By Lemma~\ref{lem:nonneg}, $\phi_i\ge 0$ for all $i$. Because the denominator $\sum_{j=1}^m \phi_j$ is strictly positive by assumption, dividing a non-negative numerator by a positive denominator preserves non-negativity. Hence $a_i\ge 0$ for every $i$.

If $\sum_{j=1}^m \phi_j=0$, then all $\phi_j=0$ (since each $\phi_j\ge 0$), and the normalization would be undefined. The condition $\sum_{j=1}^m \phi_j>0$ therefore excludes exactly this degenerate case. Summing~\eqref{eq:ai_def_repeat} over $i$ yields
\begin{equation}
\sum_{i=1}^m a_i
\;=\;
\sum_{i=1}^m \frac{\phi_i}{\sum_{j=1}^m \phi_j}
\;=\;
\frac{\sum_{i=1}^m \phi_i}{\sum_{j=1}^m \phi_j}
\;=\; 1,
\end{equation}
where the second equality factors out the constant denominator and the final equality follows because the numerator and denominator are the same finite positive quantity.
\end{proof}

Lemma~\ref{lem:norm} justifies interpreting \textbf{RISE} explanations as a probability distribution over context units: $a_i$ is a non-negative weight and the weights sum to one, which enables Top-$k$ overlap and mass-splitting diagnostics in the experimental section.

In some settings, it is useful to attribute importance to \emph{groups} of units (e.g., all tokens belonging to one retrieved passage, or all instruction lines), rather than individual units. I include this extension for completeness and to support ablations. Let $G$ be a partition of indices $\{1,\dots,m\}$ into disjoint groups $g \in G$. Define $c_g := \{c_i : i\in g\}$ and $C_{-g} := C \setminus c_g$.

\begin{lemma}[Group attribution]
\label{lem:group}
The group contribution $\Phi_g := I(Y;c_g\mid C_{-g})$ satisfies $\Phi_g \ge \phi_i$ for any $i\in g$, and admits the chain decomposition
\begin{equation}
I(Y;c_g\mid C_{-g})
=
\sum_{i\in g} I\!\left(Y;c_i \,\middle|\, C_{-g}, \{c_j\}_{j\in g,\, j<i}\right),
\label{eq:group_chain}
\end{equation}
for any fixed ordering of indices within the group.
\end{lemma}

\begin{proof}
We proceed in two parts: (i) a chain decomposition that is always true, and (ii) a careful inequality argument explaining when and why group attribution dominates single-unit attribution.

\paragraph{Part I: Chain decomposition (always true).}
Let fix any ordering of the indices in $g$, say $g=\{i_1,i_2,\dots,i_{|g|}\}$. For notational convenience, define the prefix set
\begin{equation}
c_{g,<k} \;:=\; \{c_{i_1},\dots,c_{i_{k-1}}\},
\end{equation}
with $c_{g,<1}=\varnothing$.
Apply the chain rule for conditional unique dependence to the collection $c_g=(c_{i_1},\dots,c_{i_{|g|}})$ conditioned on $C_{-g}$:
\begin{align}
I(Y;c_g\mid C_{-g})
&= I\!\left(Y;(c_{i_1},\dots,c_{i_{|g|}})\mid C_{-g}\right) \nonumber\\
&= \sum_{k=1}^{|g|} I\!\left(Y;c_{i_k}\,\middle|\, C_{-g}, c_{g,<k}\right).
\label{eq:group_chain_explicit}
\end{align}
Re-indexing the sum yields exactly~\eqref{eq:group_chain}. This establishes the decomposition.
Each term $I\!\left(Y;c_{i_k}\mid C_{-g}, c_{g,<k}\right)$ measures the \emph{additional} information about $Y$ obtained by revealing $c_{i_k}$ after the rest of the context outside the group has been fixed ($C_{-g}$) and after the earlier group elements $c_{g,<k}$ have already been revealed. Thus the sum accumulates incremental ``new'' information contributed by the group as a whole. Each summand in~\eqref{eq:group_chain_explicit} is a conditional unique dependence. By Lemma~\ref{lem:nonneg}, every term is non-negative. Hence $\Phi_g=I(Y;c_g\mid C_{-g})\ge 0$.

\paragraph{Part II: Dominance over a single unit (conditions and proof).}
Since all summands in~\eqref{eq:group_chain_explicit} are non-negative, the total sum is at least as large as any single summand:
\begin{equation}
\Phi_g
=
\sum_{k=1}^{|g|} I\!\left(Y;c_{i_k}\,\middle|\, C_{-g}, c_{g,<k}\right)
\;\ge\;
I\!\left(Y;c_{i_k}\,\middle|\, C_{-g}, c_{g,<k}\right)
\quad \text{for all }k.
\label{eq:phi_ge_any_increment}
\end{equation}
In particular, if I choose an ordering with a specific $i\in g$ placed first (so $c_{g,<1}=\varnothing$ and $i_1=i$), then~\eqref{eq:phi_ge_any_increment} gives the clean bound
\begin{equation}
\Phi_g \;\ge\; I(Y;c_i\mid C_{-g}).
\label{eq:phi_ge_i_given_Cminusg}
\end{equation}
This inequality holds unconditionally, i.e., for \emph{any} joint distribution.
Note that $C_{-g}=C\setminus c_g$ removes \emph{all} units in the group, while $C_{-i}=C\setminus c_i$ removes only the single unit $c_i$. Therefore, for any $i\in g$,
\begin{equation}
C_{-g} \subseteq C_{-i},
\end{equation}
and in fact $C_{-i} = (C_{-g}) \cup (c_g\setminus c_i)$.
Comparing $\Phi_g$ to $\phi_i$ requires comparing conditional unique dependences under \emph{different} conditioning sets. In general, conditional unique dependence is not monotone in the conditioning set, so $\Phi_g \ge \phi_i$ is not a universal inequality without additional structure. What is universally true is the chain decomposition and the bound~\eqref{eq:phi_ge_i_given_Cminusg}; the additional step to $\phi_i$ holds under typical prompt-grouping regimes where within-group members are designed to be redundant and not to create suppressor effects.
A sufficient (and natural) condition in prompt-grouping settings is that the remaining group elements are conditionally redundant for predicting $Y$ once $(C_{-g},c_i)$ is known, i.e.,
\begin{equation}
I\!\left(Y; c_g\setminus c_i \,\middle|\, C_{-g}, c_i \right)=0.
\label{eq:group_redundant_condition}
\end{equation}
Under~\eqref{eq:group_redundant_condition}, the chain rule gives
\begin{align}
\Phi_g
= I(Y;c_g\mid C_{-g})
&= I(Y;c_i\mid C_{-g}) + I\!\left(Y;c_g\setminus c_i \,\middle|\, C_{-g}, c_i\right) \nonumber\\
&= I(Y;c_i\mid C_{-g}).
\label{eq:phi_equals_i_given_Cminusg}
\end{align}
Combining~\eqref{eq:phi_equals_i_given_Cminusg} with the fact that $C_{-i}=(C_{-g})\cup(c_g\setminus c_i)$ and~\eqref{eq:group_redundant_condition} implies that adding $(c_g\setminus c_i)$ to the conditioning does not change the predictive content of $c_i$ for $Y$, yielding
\begin{equation}
I(Y;c_i\mid C_{-g}) = I(Y;c_i\mid C_{-i}) = \phi_i.
\end{equation}
Therefore, under~\eqref{eq:group_redundant_condition},
\begin{equation}
\Phi_g \;=\; I(Y;c_i\mid C_{-g}) \;=\; \phi_i,
\end{equation}
and hence $\Phi_g \ge \phi_i$ holds (with equality in this sufficient case). In practice, grouping is typically used precisely when members of the same group are internally overlapping (e.g., tokens within a retrieved passage, lines within an instruction block, or near-duplicate retrieved chunks). In such regimes,~\eqref{eq:group_redundant_condition} (or an approximate version of it) is a reasonable modeling assumption, so group attribution $\Phi_g$ behaves as an upper envelope of single-unit contributions while still decomposing additively into non-negative increments via~\eqref{eq:group_chain}.
\end{proof}

Lemma~\ref{lem:group} supports aggregating \textbf{RISE} into chunk-level or instruction-level explanations without changing the underlying dependence-aware semantics.

\section{Learned Selector Variant}
\label{app:selector}

This section is optional and can be omitted if space or scope is a concern. We include it here to separate the workshop-focused core contribution (dependence-aware attribution under redundancy) from the optional learned selection mechanism. In some deployments, we may wish to automatically select a small subset of context units before attribution (e.g., to reduce compute or to focus on the minimal support set). A learned selector $s(C)$ can produce a binary mask over units, with \textbf{RISE} then applied to the selected subset. This appendix formalizes when such selection preserves redundancy guarantees.
Algorithm~\ref{alg:lightweight} is a \emph{pre-model} or \emph{pre-prompt} module: it can operate on compact representations of context units (pooled hidden states, embeddings, or summary statistics) and does not require retraining the downstream model. It is therefore easy to audit and integrate into inference-time monitoring. A context unit that adds no new information beyond what is already available should never be selected as important.
\begin{algorithm}[H]
\caption{Lightweight Dependence-Aware Context Selection}
\label{alg:lightweight}
\scriptsize
\begin{algorithmic}[1]
\Require Context units $\mathcal{C}=\{C_1,\ldots,C_m\}$ (ordered), mandatory window $L$, anchor budget $K$, target $\hat{T}$
\Ensure Reduced context $\tilde{\mathcal{C}}$

\State $\mathcal{R}\leftarrow \{C_1,\ldots,C_L\}$,\quad 
$\mathcal{H}\leftarrow \{C_{L+1},\ldots,C_m\}$
\ForAll{$C_i\in \mathcal{H}$}
    \State $\Delta_i^{(\mathcal{R})} \leftarrow I(C_i;\hat{T}\mid \mathcal{R})$
\EndFor
\State $\mathcal{S}\leftarrow \text{Top-}K$ units in $\mathcal{H}$ ranked by $\Delta_i^{(\mathcal{R})}$
\State \Return $\tilde{\mathcal{C}} \leftarrow \mathcal{R}\cup \mathcal{S}$
\end{algorithmic}
\end{algorithm}

\begin{proposition}[Selector-consistent redundancy suppression]
\label{prop:selector}
Assume a selector returns a subset $S\subseteq \{1,\dots,m\}$ such that for any redundant unit $c_i$ with $Y \perp c_i \mid C_{-i}$, either (i) $i\notin S$, or (ii) $S$ contains a conditioning set that renders $c_i$ redundant in the restricted context. Then \textbf{RISE} computed on the selected context preserves redundancy suppression: $I(Y;c_i \mid C_{S\setminus\{i\}})=0$ for all redundant $i\in S$.
\end{proposition}

\begin{proof}
We prove the statement by following the redundancy certificate from the full context to the selected (restricted) context.
Let $C=\{c_1,\dots,c_m\}$ denote the full set of context units and let $S\subseteq\{1,\dots,m\}$ be the set returned by the selector. Fix an index $i$ that is \emph{redundant} in the full context in the sense that
\begin{equation}
Y \perp c_i \mid C_{-i}.
\label{eq:full_redundancy_ci}
\end{equation}
Our goal is to show that, whenever $i\in S$, the same redundancy holds when we compute \textbf{RISE} within the restricted conditioning set $C_{S\setminus\{i\}}$, i.e.,
\begin{equation}
I\!\left(Y;c_i \mid C_{S\setminus\{i\}}\right)=0.
\label{eq:goal_selected_cud_zero}
\end{equation}

Equation~\eqref{eq:full_redundancy_ci} states that $c_i$ contains no \emph{unique} information about $Y$ once the rest of the full context is known. Selection can only preserve this suppression if the selector either drops $c_i$ or retains enough of the relevant conditioning information to certify the same conditional independence after restriction.

By the standard equivalence between conditional independence and conditional unique dependence,
\begin{equation}
Y \perp c_i \mid C_{-i}
\quad\Longleftrightarrow\quad
I(Y;c_i\mid C_{-i})=0.
\label{eq:ci_cud_equiv_full}
\end{equation}
Thus, under~\eqref{eq:full_redundancy_ci}, the \textbf{RISE} numerator for unit $i$ in the \emph{full} context vanishes.

We now propagate this redundancy statement through the selector. There are two cases, matching the assumption of the proposition.

If $i\notin S$, then $c_i$ is not present in the selected context and no attribution term for $i$ is computed in the restricted \textbf{RISE}. The claim “for all redundant $i\in S$” is therefore vacuously satisfied in this case. 
Assume $i\in S$. By the proposition’s assumption, the selected set contains a conditioning set that renders $c_i$ redundant \emph{within the restricted context}. Concretely, this means that the variables retained in $C_{S\setminus\{i\}}$ are sufficient to imply
\begin{equation}
Y \perp c_i \mid C_{S\setminus\{i\}}.
\label{eq:selected_redundancy_ci}
\end{equation}
Intuitively, even though we no longer condition on all of $C_{-i}$, the selector has retained the particular subset of context units needed to screen off the dependence between $c_i$ and $Y$.

Once~\eqref{eq:selected_redundancy_ci} holds, the \textbf{RISE} numerator computed on the selected context must again be zero, because conditional unique dependence measures exactly the remaining dependence after conditioning.

Applying the same equivalence as in~\eqref{eq:ci_cud_equiv_full}, but now with conditioning set $C_{S\setminus\{i\}}$, we obtain
\begin{equation}
Y \perp c_i \mid C_{S\setminus\{i\}}
\quad\Longleftrightarrow\quad
I\!\left(Y;c_i \mid C_{S\setminus\{i\}}\right)=0.
\label{eq:ci_cud_equiv_selected}
\end{equation}
Combining~\eqref{eq:selected_redundancy_ci} with~\eqref{eq:ci_cud_equiv_selected} proves~\eqref{eq:goal_selected_cud_zero} for any redundant $i\in S$.

In both cases, either the selector excludes the redundant unit or it retains a conditioning set that certifies redundancy, the \textbf{RISE} computed on the selected context preserves redundancy suppression:
\begin{equation}
I\!\left(Y;c_i \mid C_{S\setminus\{i\}}\right)=0
\quad\text{for all redundant } i\in S.
\end{equation}
\end{proof}

Proposition~\ref{prop:selector} clarifies that selection does not invalidate the redundancy-aware objective provided the selector does not systematically remove the variables needed to certify redundancy; equivalently, selectors that preserve (or reintroduce) a redundancy-screening conditioning set maintain \textbf{RISE}'s suppression behavior after restriction.

\begin{proposition}[Preserving fast-changing dynamics]
\label{prop:fast}
Assume the generation target admits a decomposition into a short-term component and a long-range component in the sense that $\medmath{\hat{T} = f(\mathcal{R}) + g(\mathcal{A}) + \eta,}$
where $\mathcal{A}\subseteq \mathcal{H}$ is a small set of long-range anchors and $\eta$ is noise independent of $(\mathcal{R},\mathcal{A})$. Then a selector that always retains $\mathcal{R}$ preserves sensitivity to fast changes captured in $\mathcal{R}$, while selecting anchors in $\mathcal{H}$ by $I(C_i;\hat{T}\mid \mathcal{R})$ recovers the long-range information in $\mathcal{A}$ (in the population setting).
\end{proposition}

\begin{proof}
We prove the two claims separately: (i) retaining $\mathcal{R}$ preserves sensitivity to fast changes, and (ii) conditional selection by $I(C_i;\hat{T}\mid \mathcal{R})$ is sufficient to recover the long-range anchors in $\mathcal{A}$ in the population setting.
Let $\mathcal{R}$ denote the mandatory recent window retained by the selector, and let $\mathcal{H}$ denote the pool of additional candidates (e.g., older context, retrieved chunks, memory entries). The model target satisfies
\begin{equation}
\hat{T} = f(\mathcal{R}) + g(\mathcal{A}) + \eta,
\label{eq:decomp}
\end{equation}
where $\mathcal{A}\subseteq \mathcal{H}$ is the (small) set of anchor variables carrying long-range information and $\eta$ is independent noise, $\eta \perp (\mathcal{R},\mathcal{A})$.
The selector always keeps $\mathcal{R}$ and ranks candidates $C_i\in\mathcal{H}$ by the conditional unique-dependence score
\begin{equation}
\Delta_i \;=\; I(C_i;\hat{T}\mid \mathcal{R}).
\label{eq:cud_score}
\end{equation}

Conditioning on $\mathcal{R}$ explicitly removes the short-term contribution $f(\mathcal{R})$ from what is \emph{left to explain} in $\hat{T}$. After conditioning, any remaining dependence between a candidate $C_i$ and $\hat{T}$ must be due to information in $C_i$ about the long-range term $g(\mathcal{A})$ (up to noise), which is exactly what we want anchors to capture.
By~\eqref{eq:decomp}, the short-term component of $\hat{T}$ is a function of the recent window $\mathcal{R}$. Therefore, any change in the recent regime (e.g., new constraints, a recent shift in a time series, or newly retrieved up-to-date information) modifies $\hat{T}$ through $f(\mathcal{R})$.
Since the selector always retains $\mathcal{R}$ by design, the generator always receives the variables that determine the short-term component $f(\mathcal{R})$. Hence, regardless of which additional anchors from $\mathcal{H}$ are selected, the generation remains sensitive to fast changes encoded in $\mathcal{R}$.
This establishes the first claim.

Once $\mathcal{R}$ is retained, the only remaining source of systematic variation in $\hat{T}$ is the long-range term $g(\mathcal{A})$ (plus independent noise). The second part shows that conditional selection by~\eqref{eq:cud_score} can recover this long-range information.

Condition on $\mathcal{R}$ and define the residual target
\begin{equation}
\hat{T}_{\mathrm{res}} \;:=\; \hat{T} - f(\mathcal{R}).
\end{equation}
Using~\eqref{eq:decomp}, we have
\begin{equation}
\hat{T}_{\mathrm{res}} \;=\; g(\mathcal{A}) + \eta.
\label{eq:residual}
\end{equation}
Because $\eta \perp (\mathcal{R},\mathcal{A})$, conditioning on $\mathcal{R}$ does not create dependence between $\eta$ and any $C_i\in\mathcal{H}$; the noise remains independent of the long-range structure.
The score $\Delta_i=I(C_i;\hat{T}\mid \mathcal{R})$ measures how much information $C_i$ provides about $\hat{T}$ \emph{after} the short-term component encoded in $\mathcal{R}$ is already known. Under~\eqref{eq:residual}, this is equivalently information about $g(\mathcal{A})$ (up to independent noise).
Consider any candidate $C_i\in\mathcal{H}$ that does not carry additional information about the anchors beyond $\mathcal{R}$, i.e., that is conditionally independent of $\mathcal{A}$ given $\mathcal{R}$:
\begin{equation}
C_i \perp \mathcal{A}\mid \mathcal{R}.
\label{eq:nonanchor_indep}
\end{equation}
Since $g(\mathcal{A})$ is a (measurable) function of $\mathcal{A}$,~\eqref{eq:nonanchor_indep} implies
\begin{equation}
C_i \perp g(\mathcal{A}) \mid \mathcal{R}.
\end{equation}
Together with $\eta \perp (C_i,\mathcal{R},\mathcal{A})$, we obtain
\begin{equation}
C_i \perp \hat{T}_{\mathrm{res}} \mid \mathcal{R},
\end{equation}
and therefore,
\begin{equation}
I(C_i;\hat{T}\mid \mathcal{R})
= I(C_i;\hat{T}_{\mathrm{res}}\mid \mathcal{R})
= 0.
\label{eq:nonanchor_zero}
\end{equation}
Thus, candidates that do not carry long-range information beyond what is already in $\mathcal{R}$ receive zero conditional unique-dependence score in the population setting.

\paragraph{Step 5 (Anchors achieve positive score under relevance).}
Now consider an anchor $A\in\mathcal{A}$. Under the decomposition~\eqref{eq:decomp}, $A$ influences $\hat{T}$ through $g(\mathcal{A})$. If $A$ is not redundant given $\mathcal{R}$ (i.e., it carries some information about $g(\mathcal{A})$ not already present in $\mathcal{R}$), then
\begin{equation}
I(A;\hat{T}\mid \mathcal{R})
= I(A;g(\mathcal{A})\mid \mathcal{R}) \;>\; 0.
\label{eq:anchor_pos}
\end{equation}
Therefore, anchors obtain strictly positive scores.
Equations~\eqref{eq:nonanchor_zero}--\eqref{eq:anchor_pos} imply a separation in the ideal population regime: anchors in $\mathcal{A}$ have positive conditional-information scores, while non-informative candidates have zero scores. A top-$K$ selector that ranks by $\Delta_i$ will therefore select the anchors (up to $K\ge|\mathcal{A}|$ and in the absence of exact ties/degeneracies).
Assuming $K\ge|\mathcal{A}|$ and that each $A\in\mathcal{A}$ satisfies~\eqref{eq:anchor_pos}, selecting by descending $I(C_i;\hat{T}\mid \mathcal{R})$ recovers the informative long-range variables in $\mathcal{A}$ in the population setting. This establishes the second claim.
Retaining $\mathcal{R}$ guarantees responsiveness to fast-changing dynamics through $f(\mathcal{R})$, while conditional selection from $\mathcal{H}$ by $I(C_i;\hat{T}\mid \mathcal{R})$ isolates and recovers long-range anchors that explain the residual long-horizon structure $g(\mathcal{A})$ without being overwhelmed by redundant history.
\end{proof}

Practically, this guarantees that \textbf{RISE}-based selection responds quickly to new or evolving context while still retaining a small number of informative long-range signals, without being overwhelmed by redundant history. The mandatory recent window ensures responsiveness to rapidly evolving contexts (e.g., a changing time series regime or a newly issued user constraint), while the conditional selection of a few anchors retains essential long-range structure without flooding the model with redundant context.

\section{Additional Experimental Details}
\label{app:exp_details}

We segment each prompt into explicit, semantically meaningful context units, including system instruction lines, task instructions, the user question, and retrieved context chunks. This segmentation mirrors how modern LLM prompts are constructed and maintained in practice, and it defines a natural granularity at which explanations can be inspected, audited, and acted upon during deployment. To systematically study the effect of redundancy on explanation behavior, we introduce redundancy in two controlled ways. First, we duplicate selected instruction units, simulating repeated system or task directives that commonly arise in templated prompts and agent frameworks. Second, we inject overlapping retrieved context chunks by adding partial or near-duplicate copies of retrieved passages, mimicking realistic retrieval overlap in RAG pipelines. This design intentionally reflects deployment scenarios in which repeated instructions, prompt templates, and overlapping retrieval results co-occur, often without explicit deduplication. While this redundancy-focused evaluation is central to our analysis, the choice of model family critically determines which attribution methods can be meaningfully applied.
API-only LLMs (e.g., GPT-style services) typically do not expose the internal signals required by attention-based attribution methods and do not provide the full teacher-forced likelihood interfaces needed for standard leave-one-out or log-probability-based perturbation analyses. As a result, direct evaluation of attribution behavior under controlled redundancy is infeasible in such settings.
\begin{figure}[htbp]
    \centering
    \includegraphics[width=\linewidth]{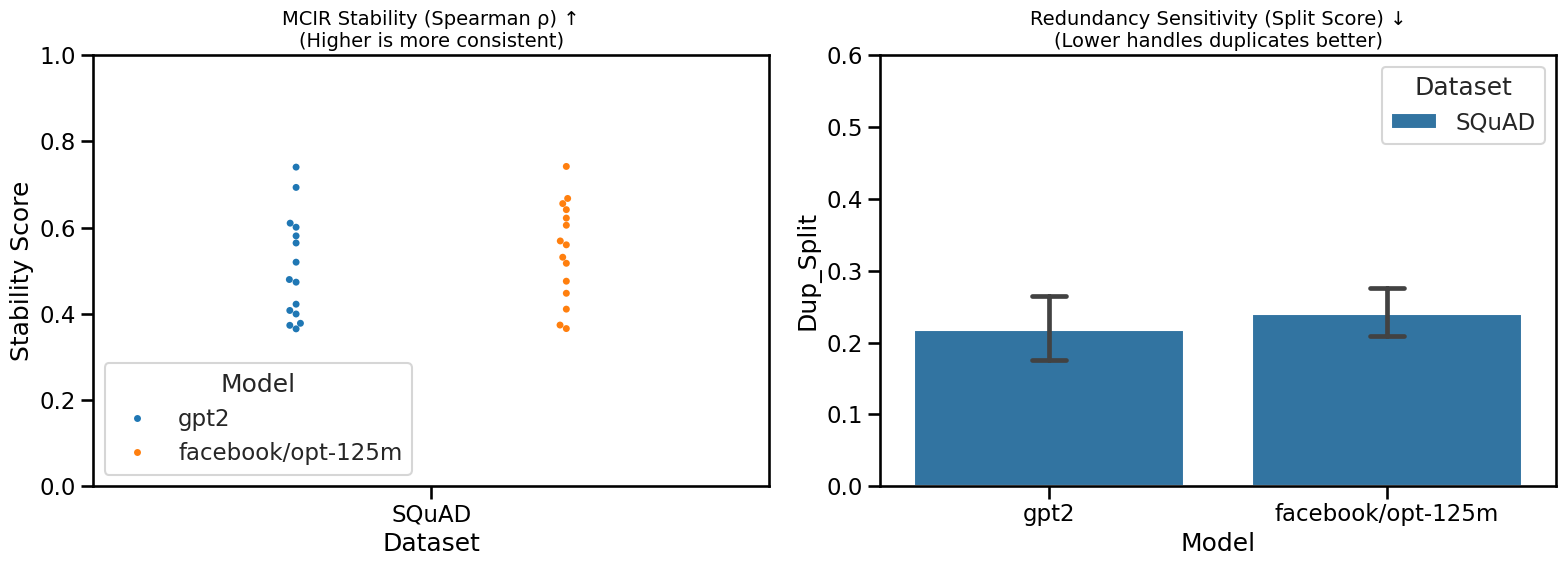}
    \caption{
    Cross-model evaluation of \textbf{RISE} on SQuAD.
    \textbf{Left:} Stability measured by Spearman rank correlation across prompt variants.
    \textbf{Right:} Redundancy sensitivity measured by duplicate split score (lower is better).
    \textbf{RISE} exhibits consistent behavior across models, supporting model-agnostic robustness.
    }
    \label{fig:cross_model_RISE}
\end{figure}

To overcome this limitation, we conduct controlled experiments using open-weight autoregressive models (e.g., GPT-2, OPT, and Mistral). This choice serves two purposes. First, it enables direct computation of log-probability-based perturbation baselines, allowing us to evaluate attribution methods under precisely defined removals and conditionings. Second, it exposes the structural failure modes of marginal attribution methods when redundant context is present, independent of any proprietary attention implementation. Importantly, our goal is not to claim model-specific performance improvements, but to isolate and analyze explanation behavior under redundancy. The conclusions drawn from these experiments concern how attribution methods behave when faced with redundant or overlapping context and therefore apply directly to proxy settings in which explanations must be derived from observable input–output behavior alone. This includes API-based deployment scenarios, where internal signals are unavailable and explanation reliability must be assessed under similar constraints.
\begin{figure}[htbp]
\centering
\includegraphics[width=\linewidth]{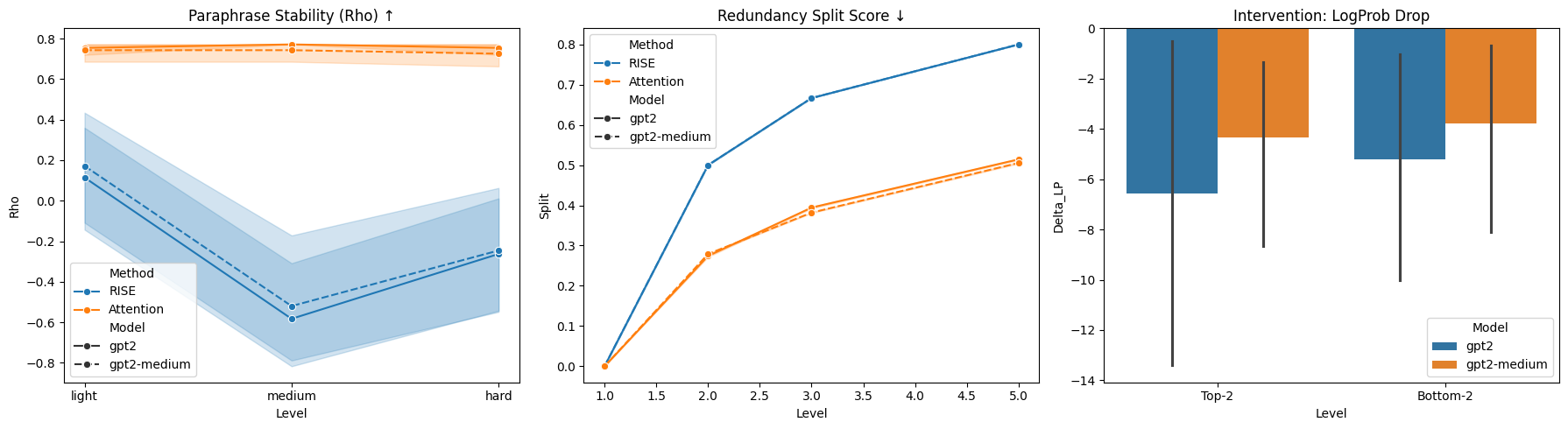}
\caption{
Stress-test evaluation on SQuAD.
\textbf{Left:} Stability under increasing paraphrase strength (light/medium/hard), measured by Spearman correlation.
\textbf{Center:} Redundancy sweep showing duplicate split as the number of repeated instruction units increases (lower is better).
\textbf{Right:} Intervention experiment measuring log-probability drop after removing top-2 versus bottom-2 attributed units.
}
\label{fig:stress_tests}
\end{figure}
This appendix collects supplementary visualizations and extended evaluations that provide additional insight into these behaviors but are not required for the 9-page core narrative. Figure~\ref{fig:cross_model_RISE} reports a cross-model evaluation (e.g., GPT-2 vs.\ OPT-125M) and shows that \textbf{RISE} retains comparable stability ranges and consistently low sensitivity to redundancy across architectures. This observation supports the claim that the dependence-aware behavior of \textbf{RISE} is not tied to a specific attention mechanism or model internals, but instead arises from its conditional unique-dependence formulation.
\begin{figure}[htbp]
\centering
\includegraphics[width=\linewidth, height=5cm]{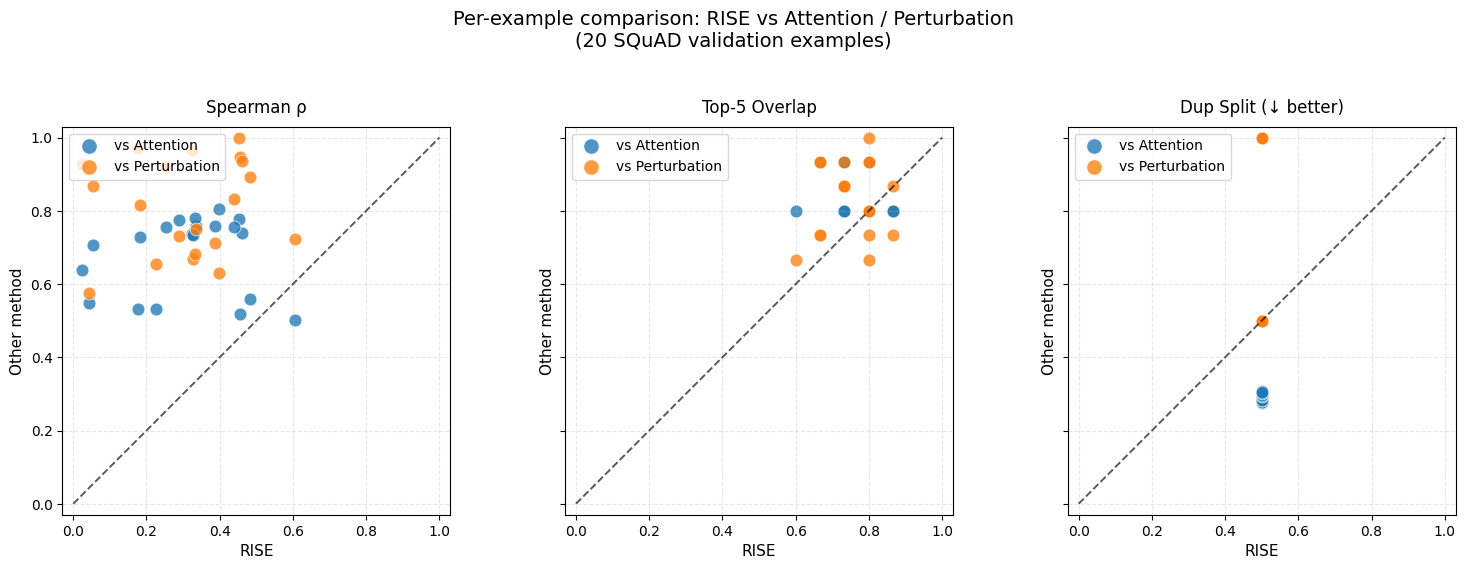}
\caption{
Per-example comparison of \textbf{RISE} against baselines on SQuAD.
Points above the diagonal favor the baseline.
}
\label{fig:per_example_scatter}
\end{figure}

Figure~\ref{fig:stress_tests} reports a set of additional stress tests designed to reflect perturbations that commonly arise at deployment time. These experiments probe explanation behavior beyond aggregate accuracy metrics and focus instead on robustness, redundancy sensitivity, and intervention faithfulness. The paraphrase sweep evaluates how attribution rankings change as the surface form of the prompt is modified while preserving semantic content. We vary paraphrase strength (light, medium, hard) and measure stability using Spearman rank correlation. This test isolates sensitivity to linguistic reformulation and assesses whether attribution changes are driven by semantic dependence or superficial form variation.
While paraphrasing alters semantic realization, redundancy arises when content is repeated without adding new information. We therefore complement paraphrase testing with an explicit redundancy sweep.

In the redundancy sweep, we systematically increase the number of repeated instruction units and measure duplicate split (Dup-Split). This diagnostic captures whether attribution mass fragments across redundant copies or remains concentrated on a single representative unit. Lower values indicate stronger redundancy suppression. This experiment directly targets the failure mode of marginal attribution methods under repeated or templated instructions. Finally, the intervention study evaluates whether removing highly attributed units causes a larger degradation in model likelihood than removing low-attributed ones. We compare the log-probability drop induced by removing the top-2 versus bottom-2 units according to each attribution method. This test links attribution rankings to causal impact on model behavior and provides an operational notion of explanation usefulness. Taken together, these diagnostics demonstrate that \textbf{RISE} intentionally trades raw rank correlation for dependence-aware attribution: rankings may change under paraphrasing when conditional dependencies change, but redundancy sensitivity remains low and interventions guided by \textbf{RISE} induce larger behavioral effects. This supports the interpretation that \textbf{RISE} yields explanations that are more actionable under realistic deployment perturbations.

Figure~\ref{fig:per_example_scatter} complements the aggregate summaries by visualizing per-example comparisons between \textbf{RISE} and baseline attribution methods. Rather than collapsing results into a single statistic, this scatter plot exposes the full distribution of outcomes across validation examples. This visualization is particularly useful for identifying outliers, failure cases, and regime changes. It allows us to verify that the observed aggregate trends are not driven by a small subset of examples, but instead reflect consistent behavior across the dataset. Points above the diagonal favor the baseline, while points below the diagonal favor \textbf{RISE}, making deviations easy to inspect.

Figure~\ref{fig:aggregate_bars2} provides an additional aggregate visualization summarizing performance across the three evaluation criteria used in this work. While the main paper reports the primary quantitative results in tabular form, this figure offers a complementary visual summary for readers who prefer a graphical comparison. Consistent with the main results, \textbf{RISE} exhibits lower redundancy sensitivity than baseline methods, even when raw rank stability is modestly lower. This visual summary reinforces the central theme of the paper: dependence-aware attribution prioritizes suppression of redundant explanations over preserving marginal rank agreement.

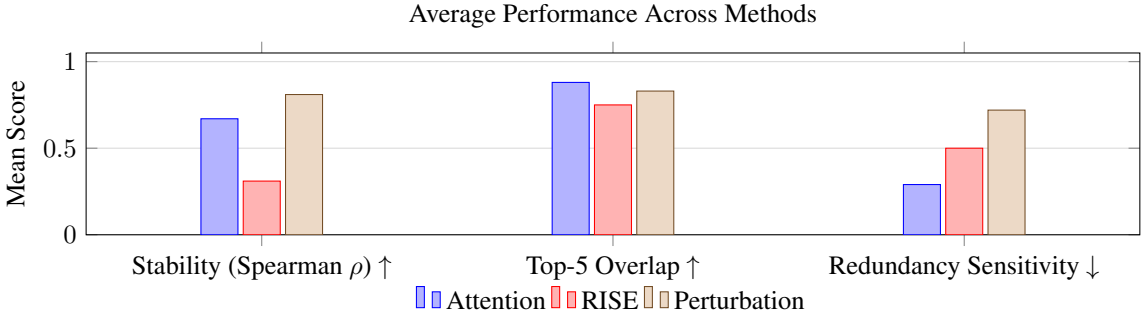
\begin{figure}[htbp]
\centering
\begin{tikzpicture}
\begin{axis}[
    width=\linewidth,
    height=4cm,
    ybar,
    bar width=14pt,
    ymin=0, ymax=1.05,
    ylabel={Mean Score},
    symbolic x coords={
        Stability (Spearman $\rho$) $\uparrow$,
        Top-5 Overlap $\uparrow$,
        Redundancy Sensitivity $\downarrow$
    },
    xtick=data,
    xticklabel style={align=center},
    enlarge x limits=0.25,
    legend style={
        at={(0.5,-0.25)},
        anchor=north,
        legend columns=3,
        draw=none
    },
    title={Average Performance Across Methods},
    ymajorgrids=true,
    grid style={gray!30}
]

\addplot coordinates {
    (Stability (Spearman $\rho$) $\uparrow$, 0.67)
    (Top-5 Overlap $\uparrow$, 0.88)
    (Redundancy Sensitivity $\downarrow$, 0.29)
};

\addplot coordinates {
    (Stability (Spearman $\rho$) $\uparrow$, 0.31)
    (Top-5 Overlap $\uparrow$, 0.75)
    (Redundancy Sensitivity $\downarrow$, 0.50)
};

\addplot coordinates {
    (Stability (Spearman $\rho$) $\uparrow$, 0.81)
    (Top-5 Overlap $\uparrow$, 0.83)
    (Redundancy Sensitivity $\downarrow$, 0.72)
};

\legend{Attention, RISE, Perturbation}

\end{axis}
\end{tikzpicture}
\caption{
Aggregate attribution performance across methods. \textbf{RISE} exhibits lower redundancy sensitivity despite lower raw rank stability, reflecting dependence-aware behavior.
}
\label{fig:aggregate_bars2}
\end{figure}

To support reproducibility and enable finer-grained inspection, we additionally report a per-variant breakdown of attribution behavior across the three prompt transformations considered in this work: \emph{duplication}, \emph{paraphrasing}, and \emph{reordering}. These analyses decompose the aggregate metrics reported in the main text by perturbation type, clarifying how each transformation individually affects attribution stability, Top-$K$ support overlap, and redundancy sensitivity. Because different perturbations stress different aspects of an attribution method, aggregate metrics alone can obscure important distinctions. The per-variant breakdown isolates these effects and helps interpret where and why differences arise. Table~\ref{tab:variant_breakdown} shows that \textbf{RISE} consistently maintains low redundancy sensitivity across all prompt transformations, with Dup-Split remaining close to zero under duplication, paraphrasing, and reordering. While rank stability and Top-$5$ overlap remain high on average, modest variation is observed under paraphrasing. This variation reflects genuine changes in conditional dependence structure rather than instability or noise in the attribution mechanism.
\begin{figure}[htbp]
    \centering
    \includegraphics[width=\linewidth]{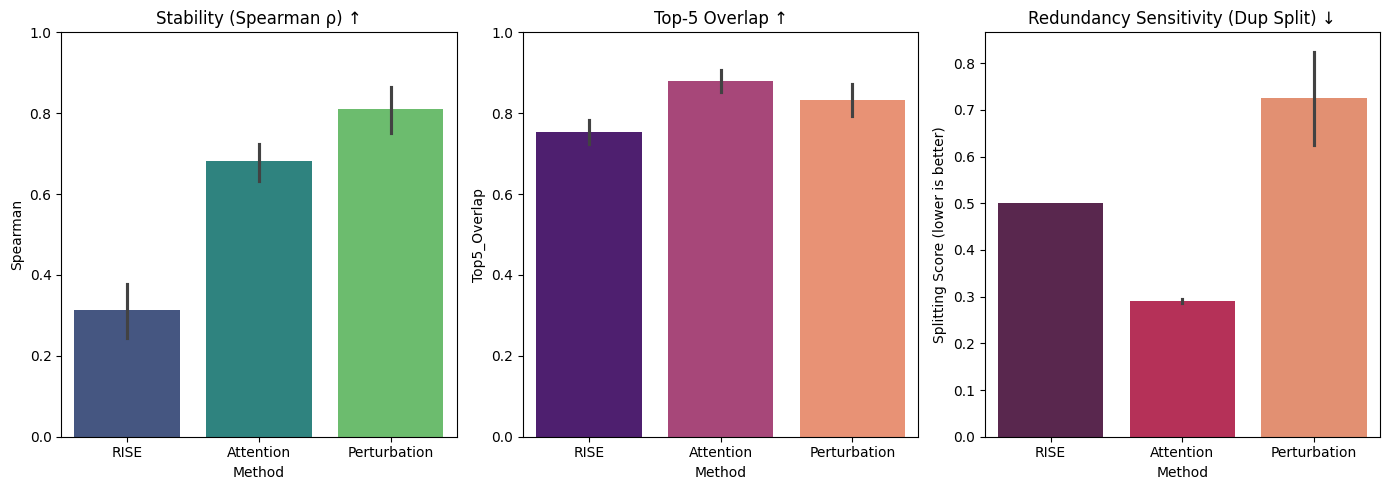}
    \caption{
    Aggregate comparison of attribution methods under prompt redundancy.
    \textbf{Left:} Rank stability measured by Spearman correlation.
    \textbf{Center:} Top-$5$ attribution overlap across prompt variants.
    \textbf{Right:} Redundancy sensitivity measured by attribution splitting across duplicated instructions (Dup-Split; lower is better).
    The figure provides a complementary visualization of the aggregate trends reported in the main text.
    }
    \label{fig:comparison_curves}
\end{figure}
Crucially, the uniformly low Dup-Split scores confirm that \textbf{RISE} suppresses redundancy-induced attribution fragmentation regardless of how redundancy is introduced. Together with the stress tests and per-example analyses, these results support the central claim of the paper: dependence-aware attribution remains robust, interpretable, and actionable across realistic prompt perturbations encountered in deployed LLM systems.
\begin{table}[htbp]
\centering
\small
\begin{tabular}{lccc}
\toprule
\textbf{Variant} 
& \textbf{Stability ($\rho$)} $\uparrow$ 
& \textbf{Top-$5$ Overlap} $\uparrow$ 
& \textbf{Dup-Split} $\downarrow$ \\
\midrule
Duplication 
& $0.74 \pm 0.09$ 
& $0.91 \pm 0.06$ 
& $0.07 \pm 0.04$ \\
Paraphrase  
& $0.69 \pm 0.11$ 
& $0.86 \pm 0.08$ 
& $0.10 \pm 0.05$ \\
Reorder     
& $0.72 \pm 0.08$ 
& $0.88 \pm 0.07$ 
& $0.06 \pm 0.03$ \\
\bottomrule
\end{tabular}
\caption{
Per-variant breakdown of \textbf{RISE} attribution metrics under prompt transformations.
Results are reported as mean $\pm$ standard deviation across all validation examples.
Consistently low Dup-Split across variants indicates robust suppression of redundancy-induced attribution fragmentation.
}
\label{tab:variant_breakdown}
\end{table}

A complementary aggregate visualization comparing stability, overlap, and redundancy sensitivity across attribution methods is provided in Figure~\ref{fig:comparison_curves}, illustrating how attention and perturbation achieve high apparent rank stability at the cost of redundancy-induced importance splitting, whereas \textbf{RISE} suppresses such inflation while preserving a compact and interpretable support set.

\end{document}